\documentclass{article}

\PassOptionsToPackage{numbers,sort&compress}{natbib}
\usepackage[preprint]{neurips_2026}

\usepackage[utf8]{inputenc}
\usepackage[T1]{fontenc}
\usepackage{hyperref}
\usepackage{url}
\usepackage{booktabs}
\usepackage{amsfonts,amsmath,amssymb,mathtools}
\usepackage{amsthm}
\usepackage{bm}
\usepackage{nicefrac}
\usepackage{microtype}
\usepackage{xcolor}
\usepackage{graphicx}
\usepackage{enumitem}
\usepackage{array}
\usepackage{tikz}
\usetikzlibrary{arrows.meta,positioning,fit,shapes.geometric}
\graphicspath{{figures/}}

\hypersetup{colorlinks=true, linkcolor=blue!60!black, citecolor=blue!60!black, urlcolor=blue!60!black}

\newtheorem{definition}{Definition}
\newtheorem{theorem}{Theorem}
\newtheorem{proposition}{Proposition}
\newtheorem{corollary}{Corollary}

\numberwithin{equation}{section}

\newcommand{\KL}{\operatorname{KL}}
\newcommand{\E}{\mathbb{E}}
\newcommand{\R}{\mathbb{R}}
\newcommand{\dd}{\mathrm{d}}
\newcommand{\Law}{\mathcal{L}}
\newcommand{\Path}{\mathcal{X}}
\newcommand{\pdata}{\pi_{\mathrm{data}}}
\newcommand{\pbase}{\pi_0}
\newcommand{\ptarget}{\pi_1}

\newcommand{\scrR}{\mathsf{R}}
\newcommand{\scrP}{\mathsf{P}}
\newcommand{\scrQ}{\mathsf{Q}}

\newcommand{\eps}{\varepsilon}
\newcommand{\F}{\bm F}
\newcommand{\J}{\bm J}

\newcommand{\Var}{\operatorname{Var}}
\newcommand{\Tr}{\operatorname{Tr}}
\newcommand{\Id}{\operatorname{Id}}

\title{Bridge Graphical Models: Coupling, Projection, and Current-Preserving Dynamics for Generative Modeling}

\author{%
Tiantian ZHANG\\
Columbia University\\
New York, NY, USA\\
\texttt{t.zhang8@columbia.edu}%
}

\begin{document}
\maketitle

\begin{abstract}
Continuous-time generative models are often built from endpoint-conditioned bridges, but generation requires a different object: a non-anticipative Markov decoder that only observes the current state and time. We identify this bridge-to-decoder compression as a structural bottleneck shared by diffusion models, flow matching, rectified flow, Schr\"odinger bridges, and field-based generative models. We introduce the \emph{Markovization gap}, the time-integrated conditional variance of the bridge velocity given the Markov state. It is the MMSE of predicting endpoint-conditioned motion from the information available to a sampler, and it measures an irreducible loss incurred before any neural network is trained. To make this bottleneck comparable across model families, we define \emph{Bridge Graphical Models} (BGMs), which separate endpoint coupling, bridge law, Markovian projection, and current-preserving dynamics representation as independent design choices. The same formalism also represents Poisson and electrostatic models as field-line bridge kernels with a corresponding field-line Markovization gap. Across synthetic, latent, and pixel-space pilots on CIFAR-10 and Fashion-MNIST, a feature-space proxy gap estimated in minutes before training ranks design choices in the same direction as downstream training loss and FID under fixed architecture, bridge, sampler, and compute. These results support the Markovization gap as a pre-training diagnostic for bridge and coupling design.
\end{abstract}

\section{Introduction}

Continuous-time generative modeling is often organized by sampler family. Diffusion models reverse noising Markov chains and are trained through variational bounds \citep{sohldickstein2015nonequilibrium,ho2020ddpm,kingma2021vdm}. Score-based SDEs give reverse-time SDE samplers and probability-flow ODEs \citep{song2021score}. Flow matching and rectified flow learn ODE velocities by regressing conditional paths \citep{lipman2023flow,liu2022rectified,tong2024cfm}. Schr\"odinger bridges optimize entropy-regularized path measures \citep{leonard2014survey,tong2024sf2m,shi2023dsbm,tang2026foundations}. Poisson-flow and electrostatic field-matching methods embed data into an augmented space and follow fields generated by charges or interactions \citep{xu2022pfgm,xu2023pfgmpp,kolesov2025field,manukhov2025ifm,shlenskii2026duality}. This sampler-centered view is useful, but it hides a more basic constraint: the bridge used for training can condition on endpoints, while the sampler used for generation cannot.

This paper studies that compression step. At training time, a bridge may know both the base endpoint and the data endpoint. At generation time, a Markov decoder only sees $(X_t,t)$. If many endpoint pairs pass through similar intermediate states with different velocities, then no Markov decoder can recover the missing endpoint information, regardless of network capacity. This error is not an optimization artifact or a weak-architecture effect; it is a property of the chosen coupling and bridge law. We call the resulting irreducible error the \emph{Markovization gap},
\begin{equation}
    \mathfrak G(\scrQ)
    =
    \int_0^1
    \E_{\scrQ}\bigl[\Tr\,\Var(U_t\mid X_t)\bigr] \,\dd t .
\end{equation}
Equivalently, $\mathfrak G(\scrQ)$ is the time-integrated MMSE of predicting the endpoint-conditioned bridge velocity $U_t$ from the Markov state. It is a process-side information diagnostic: it measures how much endpoint-conditioned local motion survives the Markov bottleneck before neural training begins.

We define \emph{Bridge Graphical Models} (BGMs) to make this bottleneck comparable across methods. A BGM separates four choices: an endpoint coupling, a bridge kernel, a Markovian projection, and a current-preserving dynamics representation. Under this view, diffusion fixes a noisy encoder bridge and learns a reverse Markov decoder; rectified flow fixes a deterministic bridge and learns its Markov projection; Schr\"odinger methods optimize the bridge under an entropic reference; field methods define the bridge through a potential and a hitting map. The decomposition is useful because it turns a qualitative design question into a measurable one: before training, how much of the bridge velocity is unrecoverable from the Markov state?

The same view clarifies the relation to variational generative modeling. A VAE uses an encoder to infer latent variables that make generation easier. Diffusion models can be read as path-space VAEs with a fixed noisy encoder bridge. More general BGMs expose additional design choices: the endpoint coupling is an inference distribution over base-data pairings, the bridge law is a path posterior, and the Markov decoder is the sampler obtained after compressing endpoint-conditioned information into non-anticipative dynamics. Coupling and bridge design are therefore not auxiliary implementation choices; they determine how much information the Markov decoder must reconstruct.

\paragraph{Contributions.}
\begin{enumerate}[leftmargin=2em,nosep]
    \item We identify Markov compression as a structural bottleneck in continuous-time generative modeling and formalize it through Bridge Graphical Models.
    \item We introduce the Markovization gap, an MMSE-style information diagnostic that measures how much endpoint-conditioned bridge velocity is unrecoverable from the Markov state. We prove that it is the floor of the flow-matching loss and that the excess above this floor controls terminal Wasserstein error under standard regularity assumptions.
    \item We formalize Poisson/electrostatic flows as field-line bridge kernels with an associated field-line Markovization gap, placing field-based methods in the same graphical-model language as diffusion and flow methods.
    \item Across synthetic, latent, CIFAR-10, and Fashion-MNIST pilots, we show that a minutes-scale proxy gap ranks coupling and bridge choices in the same direction as downstream training loss and FID under fixed compute.
\end{enumerate}

\paragraph{Why Concept \& Feasibility.}
The scope of BGMs is broader than what can be exhaustively validated in one submission: the framework suggests a common design language for diffusion, flow, bridge, and field-based generative modeling. We therefore focus on the central feasibility claim that the Markovization gap can be estimated before training and can rank bridge/coupling choices that are easier to compress into a Markov decoder. The experiments test this claim across synthetic, latent, and pixel-space settings while keeping architecture, bridge, sampler, and compute fixed within each comparison.

\section{Bridge graphical models}
\label{sec:bgm}

Let $\mathcal M$ be a state space, often $\R^d$ or an augmented space $\R^d\times\R^D$, and let $\Path=C([0,1],\mathcal M)$ with coordinate process $X_t$. Let $\pbase$ be a simple base distribution and $\ptarget=\pdata$ the data distribution.

\begin{definition}[Endpoint coupling]
An endpoint coupling is $\Gamma(\dd z,\dd x)\in\Pi(\pbase,\ptarget)$. Examples include independent coupling, minibatch OT, entropic OT, paired-data coupling, or a learned encoder coupling $\Gamma_\phi(\dd z,\dd x)=q_\phi(\dd z\mid x)\pdata(\dd x)$.
For learned encoders this is an exact element of $\Pi(\pbase,\ptarget)$ only when the aggregated posterior $\int q_\phi(\dd z\mid x)\pdata(\dd x)$ equals $\pbase$; otherwise it should be viewed as an approximate coupling, or the effective base distribution should be taken to be the aggregated encoder marginal.
\end{definition}

\begin{definition}[Bridge kernel and reference path law]
For each endpoint pair $(z,x)$, a bridge kernel is a path-space probability measure $\scrR^{z,x}_\lambda(\dd X_{0:1})$ with $X_0=z$ and $X_1=x$. The aggregated reference path law is
\begin{equation}
    \scrQ_{\Gamma,\lambda}(\dd X_{0:1})
    =\int \scrR^{z,x}_\lambda(\dd X_{0:1})\,\Gamma(\dd z,\dd x).
    \label{eq:reference_path_law}
\end{equation}
\end{definition}

\begin{definition}[Bridge graphical model]
A BGM is a tuple $(\Gamma_\phi,\scrR_\lambda,\scrP_\theta,p_\psi)$: an endpoint inference coupling, a bridge kernel, a non-anticipative Markov path law initialized from $\pbase$, and an optional observation decoder. Let
\begin{equation}
    \scrQ_{\phi,\lambda}(\dd Y_{0:1}\mid x)
    =\int \scrR^{z,x}_\lambda(\dd Y_{0:1})\,q_\phi(\dd z\mid x)
\end{equation}
denote the conditional inference path law induced by the encoder coupling and bridge kernel. Formally,
\begin{equation}
\log p_{\theta,\psi}(x)\ge
\E_{Y_{0:1}\sim\scrQ_{\phi,\lambda}(\cdot\mid x)}\left[
\log p_\psi(x\mid Y_1)-\log\frac{\dd\scrQ_{\phi,\lambda}(\cdot\mid x)}{\dd\scrP_\theta}(Y_{0:1})\right],
\label{eq:path_elbo}
\end{equation}
whenever the inference path law is absolutely continuous with respect to the generative path law and the decoder likelihood is integrable.
\end{definition}

Equation~\eqref{eq:path_elbo} is a continuous-depth VAE: the latent variable is the whole path. It should be read as a unifying variational template, not as a claim that every BGM instance admits a practical exact-likelihood implementation. For exact endpoint-conditioned bridges, the literal path KL on the closed interval $[0,1]$ is generally singular in a continuous state space: the inference bridge enforces $Y_1=x$, whereas an unconstrained Markov decoder assigns zero probability to that point event. A finite variational objective therefore requires either evaluating the path KL on a truncated interval $[0,1-\delta]$ with the last transition absorbed into a terminal likelihood such as $p_\psi(x\mid Y_{1-\delta})$, or relaxing the exact terminal constraint with a non-degenerate observation decoder such as a noisy $p_\psi(x\mid Y_1)$. The exact bridge-matching objectives used below should therefore be read as matching or optimal-control relaxations of the path-space ELBO, not as literal finite-KL objectives on the closed interval.

In shared-covariance stochastic representations, Girsanov's theorem reduces the path KL to an integrated drift-matching cost \citep{oksendal2003sde}. In vanishing-noise deterministic representations, exact absolute continuity typically fails and the operational objective becomes a flow-matching or optimal-control relaxation. In field-line representations, the natural computational object may be a hitting map or a small-noise bridge rather than a tractable path density. Thus the ELBO clarifies the graphical-model factorization, while the concrete training loss depends on the chosen dynamics representation.

Figure~\ref{fig:bgm} shows the resulting computation graph.

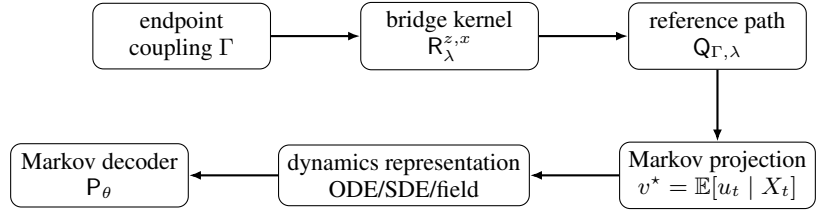
\begin{figure}[t]
\centering
\begin{tikzpicture}[
    node distance=0.95cm and 1.18cm,
    box/.style={draw, rounded corners, align=center, minimum height=0.72cm, minimum width=2.35cm, font=\small},
    arrow/.style={-{Latex[length=1.6mm]}, thick}
]
    \node[box] (coupling) {endpoint\\coupling $\Gamma$};
    \node[box, right=of coupling] (bridge) {bridge kernel\\$\scrR^{z,x}_\lambda$};
    \node[box, right=of bridge] (reference) {reference path\\$\scrQ_{\Gamma,\lambda}$};
    \node[box, below=of reference] (projection) {Markov projection\\$v^\star=\E[u_t\mid X_t]$};
    \node[box, left=of projection] (representation) {dynamics representation\\ODE/SDE/field};
    \node[box, left=of representation] (decoder) {Markov decoder\\$\scrP_\theta$};
    \draw[arrow] (coupling) -- (bridge);
    \draw[arrow] (bridge) -- (reference);
    \draw[arrow] (reference) -- (projection);
    \draw[arrow] (projection) -- (representation);
    \draw[arrow] (representation) -- (decoder);
\end{tikzpicture}
\caption{The BGM decomposition. Existing models usually fix several boxes implicitly; the framework makes coupling, bridge, projection, and dynamics representation explicit and therefore comparable.}
\label{fig:bgm}
\end{figure}

\subsection{Current-preserving dynamics representation}
\label{sec:dynamics_representation}

The fourth BGM component is a representation choice rather than a new probability path. Let $\rho_t$ solve $\partial_t\rho_t+\nabla\cdot(\rho_t v_t)=0$ with current $\J_t=\rho_t v_t$. For an SDE $\dd X_t=b_t(X_t)\dd t+\sigma_t(X_t)\dd B_t$ with diffusion covariance $a(x,t)=\sigma_t(x)\sigma_t(x)^\top$, the Fokker--Planck current is
\begin{equation}
    J_{t,i}=\rho_t b_{t,i}-\frac12\sum_j\partial_j(a_{ij}(x,t)\rho_t).
\end{equation}

\begin{theorem}[Current-preserving stochastic realization]
\label{thm:current_realization}
Let $\rho_t$ be positive and sufficiently smooth on the support of interest, and suppose the continuity equation with velocity $v_t$ is well posed with suitable decay or boundary conditions. Given a smooth diffusion matrix $a(x,t)=[a_{ij}(x,t)]\succeq0$, define
\begin{equation}
    b_i^{(a)}(x,t)=v_i(x,t)+\frac{1}{2\rho_t(x)}\sum_j\partial_j(a_{ij}(x,t)\rho_t(x)).
    \label{eq:current_preserving_drift}
\end{equation}
Then the SDE with drift $b^{(a)}$ and diffusion covariance $a(x,t)$ has the same marginal path $\rho_t$ whenever the corresponding Fokker--Planck equation is well posed and has a unique solution. If $a(x,t)=g(t)^2I$, then $b^{(a)}=v_t+\frac{g(t)^2}{2}\nabla\log\rho_t$.
\end{theorem}

This identity is standard, but it is useful for separating design choices. The bridge and projection define a probability current; the dynamics representation decides whether that current is sampled as an ODE, an SDE with score correction, or an augmented field process. The theorem is a regularity-dependent representation identity, not a claim that score estimation is automatically well conditioned.

\subsection{Positioning relative to existing unifications}
\label{sec:related}

Stochastic interpolants unify deterministic flows and stochastic diffusions through interpolating processes and probability currents \citep{albergo2023flows,albergo2023stochastic}. Diffusion Flow Matching describes generation as choosing a coupling and a deterministic or stochastic bridge, then learning a Markovian projection \citep{silveri2024dfm}. Generator Matching gives a broad Markov-process view: choose a probability path, construct conditional generators for data-conditioned paths, and learn the marginal infinitesimal generator with scalable Bregman losses \citep{holderrieth2025generator}. Recent statistical notes further develop Markovian generative modeling and generator matching from a probability-path, generator, and stability viewpoint \citep{aamari2026statistical}. Schr\"odinger bridge work studies entropy-minimizing path measures, stochastic control, reciprocal processes, and Markov projections \citep{follmer1988random,leonard2014survey,tang2026foundations}. Recent algorithms relate flow matching, OT flow matching, Schr\"odinger bridge flow matching, DSBM, and proportional Markovian fitting \citep{kim2025unified,kholkin2026ipmf}. Latent stochastic interpolants add a learned encoder/decoder and continuous-time ELBO \citep{singh2025lsi}. Field-matching duality studies when interaction-field models and conditional flow matching describe equivalent dynamics \citep{shlenskii2026duality}.

BGM is compatible with these frameworks, but asks a different diagnostic question. Generator Matching and related Markovian frameworks specify or learn a Markov generator that realizes a marginal probability path. BGM instead starts from the endpoint-coupled bridge law and asks how much endpoint-conditioned information is lost when that bridge is compressed into a non-anticipative Markov decoder. The Markovization gap is not a solver for Schr\"odinger bridges, a replacement for stochastic interpolants, or a competing generator-matching objective. It is a scalar measure of the residual information discarded by the conditional-to-Markov projection step that these objectives often perform implicitly.

\paragraph{Diagnostics for generative models.}
Recent work has also proposed diagnostic quantities for generative modeling rather than new samplers. Jalali et al.~\citep{jalali2023rke} use R\'enyi Kernel Entropy to evaluate diversity and effective mode coverage of generated samples. Our diagnostic is complementary: it does not score the final sample distribution, but measures a process-side bottleneck before training, namely the irreducible error incurred when an endpoint-conditioned bridge is compressed into a Markov decoder.

\begin{table}[t]
\centering
\small
\caption{Positioning relative to active unification directions.}
\label{tab:positioning}
\begin{tabular}{@{}p{0.25\linewidth}p{0.32\linewidth}p{0.36\linewidth}@{}}
\toprule
Framework & Main primitive & BGM addition \\
\midrule
Stochastic interpolants & interpolating process and current & endpoint inference coupling and path-space graphical model \\
DFM / bridge matching & coupling, bridge, Markov projection & projection gap as bridge-selection diagnostic \\
Generator Matching / Markovian GM & probability path and infinitesimal generator & bridge-to-Markov compression gap as a pre-training design diagnostic \\
IPMF / SB fitting & alternating reciprocal and Markov fitting & information-loss score for the Markovization step \\
Latent stochastic interpolants & latent ELBO and encoder/decoder & bridge law and dynamics representation as independent choices \\
PFGM / EFM / IFM & augmented electrostatic/interaction field & hitting-map bridge kernels and field-line Markovization gap \\
RKE-style diagnostics & output-distribution diversity & complementary process-side compressibility diagnostic \\
\bottomrule
\end{tabular}
\end{table}

\section{Markovian projection and gap}
\label{sec:markovization}

Assume $\scrQ$ has a local velocity target $u_t$ in the following deterministic weak sense: $t\mapsto\rho_t=(X_t)_\#\scrQ$ admits a finite vector-valued current $\J_t$ satisfying
\begin{equation}
    \frac{\dd}{\dd t}\int f(x)\rho_t(\dd x)
    =
    \int \nabla f(x)\cdot \J_t(\dd x)
\end{equation}
for smooth compactly supported test functions, and $\J_t$ is absolutely continuous with respect to $\rho_t$ with square-integrable density $u_t$ along the reference paths. For absolutely continuous deterministic bridges, this reduces to $u_t=\partial_tX_t$ for almost every $t$. All conditional expectations below are therefore understood up to $\rho_t$-null sets and $dt$-null time sets.

A Markov ODE decoder $\dd Y_t=v_\theta(Y_t,t)\dd t$ cannot condition on both endpoints; it only sees $(Y_t,t)$. This non-anticipative projection is the population operation behind flow matching and is closely related to classical Markovian mimicking/projection results for It\^o processes \citep{gyongy1986mimicking,brunick2013mimicking}. The assumptions needed for such mimicking results are not cosmetic: in practical generative models, data may concentrate near low-dimensional manifolds, learned fields may be non-Lipschitz, and score estimates may be singular near regions of low density. The gap defined below is therefore a population diagnostic under a chosen regularization, not a guarantee that arbitrary neural samplers are well posed.

\begin{definition}[Markovization gap]
\label{def:gap}
Define the population matching loss
\begin{equation}
    \mathcal L_{\mathrm{FM}}(v)
    :=\int_0^1\E_{\scrQ}\|u_t-v(X_t,t)\|^2\dd t .
    \label{eq:fm_loss_def}
\end{equation}
The Markovization gap is its infimum over Markov vector fields,
\begin{equation}
    \mathfrak G(\scrQ)
    =\inf_v \mathcal L_{\mathrm{FM}}(v).
    \label{eq:markov_gap_def}
\end{equation}
\end{definition}

\begin{proposition}[Projection identity]
\label{prop:projection}
Assume $\int_0^1\E_{\scrQ}\|u_t\|^2\dd t<\infty$. The population minimizer of the matching loss is
\begin{equation}
    v^\star(x,t)=\E_{\scrQ}[u_t\mid X_t=x],
\end{equation}
and
\begin{equation}
    \mathfrak G(\scrQ)=\int_0^1\E_{\scrQ}\Tr\Var(u_t\mid X_t)\dd t .
    \label{eq:gap_variance}
\end{equation}
\end{proposition}

The gap depends on the bridge/coupling rather than on neural capacity. Low gap means that the chosen endpoint-conditioned bridge is almost Markov after observing $X_t$; high gap means that multiple endpoint configurations induce conflicting velocities at the same location and time. It is not a replacement for approximation, optimization, or sampler-error analysis: it isolates only the irreducible compression error created before a finite neural function class is chosen.

\subsection{Information-theoretic view: the Markov bottleneck}
\label{sec:markov_bottleneck}

The Markovization gap can be interpreted as an information bottleneck quantity, but in a task-specific squared-loss sense rather than as a generic Shannon entropy. Let $H=(Z,X)$ denote the endpoint information used to define the bridge and let $S_t=(X_t,t)$ denote the information available to a Markov sampler at time $t$. The bridge velocity $U_t$ may depend on the hidden endpoint information $H$, whereas the sampler can only use $S_t$. The Markov projection
\begin{equation}
    v^\star(S_t)=\E[U_t\mid S_t]
\end{equation}
is the Bayes estimator of the bridge velocity under the information restriction imposed by sampling. The Markovization gap is exactly the time-integrated MMSE of this bottleneck:
\begin{equation}
    \mathfrak G(\scrQ)
    =
    \int_0^1
    \E\|U_t-\E[U_t\mid S_t]\|^2\,\dd t .
\end{equation}

This separates two questions that are often conflated. A bridge may define a smooth or geometrically appealing path, but if the current state $S_t$ does not retain enough information to infer the hidden endpoint-conditioned velocity, the bridge is hard to realize by a Markov decoder. Conversely, a low-gap bridge is one for which the Markov state is nearly sufficient for the local velocity. Thus $\mathfrak G$ measures neither path length nor sample quality directly. It measures the information about endpoint-conditioned motion that survives the Markov bottleneck.

For comparisons within a fixed feature space and estimator, one may also report the normalized unexplained velocity fraction
\begin{equation}
    \eta(\scrQ)
    =
    \frac{
    \int_0^1 \E\Tr\Var(U_t\mid S_t)\,\dd t
    }{
    \int_0^1 \E\Tr\Var(U_t)\,\dd t
    } .
\end{equation}
Small $\eta$ means that the Markov state explains most of the bridge velocity variance; large $\eta$ means that much of the bridge information remains hidden in the endpoints. In conditionally Gaussian settings, the residual covariance also upper bounds the conditional entropy of the unexplained velocity up to the usual Gaussian maximum-entropy relation. We therefore view the Markovization gap as a process-side information diagnostic specialized to the prediction problem that a Markov generative sampler must solve.

The next statement is the standard coupling, Cauchy--Schwarz, and Gronwall stability argument used in analyses of deterministic flow matching and stochastic interpolants \citep{benton2023error,albergo2023flows}; here we use it only to separate the irreducible projection floor $\mathfrak G(\scrQ)$ from the excess approximation error $\epsilon^2$.

\begin{theorem}[Terminal stability from projection error]
\label{thm:stability}
Let $\rho_t$ be transported by $v^\star$, and let $\rho^\theta_t$ be transported by $v_\theta$ from the same initial law. Assume both ODEs admit unique absolutely continuous solutions on $[0,1]$, $v_\theta(\cdot,t)$ is uniformly $L$-Lipschitz, and
\begin{equation}
    \epsilon^2=\int_0^1\E_{X_t\sim\rho_t}\|v_\theta(X_t,t)-v^\star(X_t,t)\|^2\dd t<\infty .
\end{equation}
Then
\begin{equation}
    W_2(\rho^\theta_1,\rho_1)\le e^L\epsilon .
    \label{eq:stability_bound}
\end{equation}
Moreover the excess matching loss equals the squared projection error:
$\mathcal L_{\mathrm{FM}}(v_\theta)-\mathfrak G(\scrQ)=\epsilon^2$.
\end{theorem}

Thus $\mathfrak G(\scrQ)$ is a floor for the bridge-matching loss, while the part of the regression loss above this floor controls marginal error under the stated population assumptions. The theorem deliberately separates a bridge-design quantity from finite-sample learning and numerical-integration errors.

\section{Empirical diagnostics and end-to-end compressibility}
\label{sec:toy}

We test the prediction that lower estimated Markovization gap identifies bridge/coupling choices that are easier to compress into a Markov decoder. Each comparison keeps the bridge, architecture, sampler, and compute fixed, so that the measured differences can be attributed to the BGM design choice under study rather than to model scale or tuning.

\paragraph{Synthetic diagnostic.}
Base samples are drawn from $\mathcal N(0,I_2)$; target samples are an eight-Gaussian ring. The bridge is fixed to $X_t=(1-t)Z+tX$, $u_t=X-Z$. We compare independent endpoint pairing to minibatch OT assignment. At 19 times in $[0.05,0.95]$, $\E\Tr\Var(u_t\mid X_t)$ is estimated by a nearest-neighbor conditional-variance estimator; this finite-sample estimate is used only as a diagnostic score, not as an exact lower bound.

\begin{figure}[t]
\centering
\begin{minipage}[t]{0.53\linewidth}
\centering
\includegraphics[width=\linewidth]{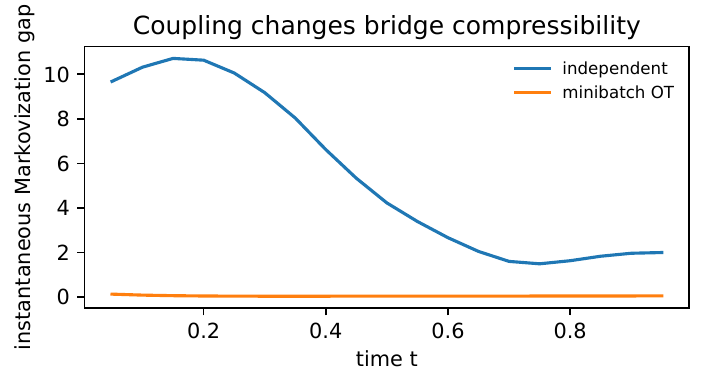}
\end{minipage}\hfill
\begin{minipage}[t]{0.42\linewidth}
\centering
\includegraphics[width=\linewidth]{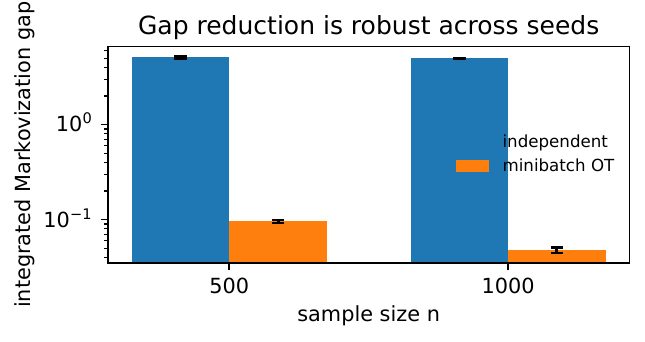}
\end{minipage}
\caption{Synthetic Markovization-gap diagnostics. Left: one run with $n=1000$ gives integrated gap $4.877$ for independent pairing and $0.038$ for minibatch OT. Right: five-seed replication at $n\in\{500,1000\}$; OT consistently lowers the irreducible compression floor by about two orders of magnitude.}
\label{fig:toy}
\end{figure}

\paragraph{End-to-end latent Markov decoder.}
We next test end-to-end compressibility: after choosing the coupling and straight bridge, we train the same neural Markov decoder and evaluate terminal samples produced by integrating its ODE. We use two low-compute real-data latent tasks: scikit-learn handwritten digits and the no-download LFW face subset bundled with scikit-image. Images are embedded into 16-dimensional PCA-whitened latents, paired with standard normal base samples, and used to train the same MLP velocity model for 500 gradient steps. We then generate fresh terminal latents with explicit Euler integration and compare them with held-out target latents. This tests whether lower Markovization gap predicts easier end-to-end Markov compression under fixed architecture and compute.

\begin{table}[t]
\centering
\small
\caption{End-to-end latent Markov-decoder benchmark, mean $\pm$ standard deviation over three seeds. The ODE is sampled with 32 Euler steps. MMD is a biased RBF latent two-sample metric with the median-distance bandwidth heuristic; 1NN accuracy is ideal near $0.5$. Lower is better except for the interpretive target of 1NN accuracy. The gap column is a finite-sample $k$NN plug-in estimate, so it should not be compared as an exact lower bound to the neural MSE column.}
\label{tab:end_to_end_benchmark}
\begin{tabular}{@{}llcccc@{}}
\toprule
Dataset & Coupling & estimated gap AUC $\downarrow$ & neural eval MSE $\downarrow$ & terminal MMD$^2$ $\downarrow$ & 1NN acc. \\
\midrule
Digits & independent & $24.530\pm2.043$ & $21.503\pm2.187$ & $0.01101\pm0.00372$ & $0.821\pm0.005$ \\
Digits & minibatch OT & $12.068\pm1.637$ & $7.178\pm0.787$ & $0.00536\pm0.00149$ & $0.834\pm0.013$ \\
LFW & independent & $23.701\pm0.880$ & $10.265\pm0.133$ & $0.05628\pm0.00511$ & $0.890\pm0.032$ \\
LFW & minibatch OT & $11.871\pm0.201$ & $3.149\pm0.272$ & $0.03074\pm0.00683$ & $0.894\pm0.010$ \\
\bottomrule
\end{tabular}
\end{table}

Table~\ref{tab:end_to_end_benchmark} supports the use of $\mathfrak G$ as a compressibility diagnostic, with the finite-sample qualification discussed below. Under the same bridge, architecture, optimizer, and sampling budget, minibatch OT roughly halves the estimated gap, reduces velocity-regression error by about $67\%$--$69\%$, and reduces terminal latent MMD by about $45\%$--$51\%$. The 1NN two-sample statistic remains far from $0.5$, which shows that this benchmark is not a visual-quality test. It nevertheless verifies the quantity targeted by the theory: lower estimated gap predicts a lower-loss Markov decoder and better latent distribution matching under fixed compute.

The gap column is a $k$NN-smoothed plug-in score, not a certified lower bound for the finite empirical regression loss. In 16-dimensional PCA space with small $n$, exact empirical bridge paths rarely collide, while nearest-neighbor neighborhoods are broad; the estimator can therefore oversmooth the conditional mean and partially measure marginal velocity variance or transport cost. 

\paragraph{CIFAR-10 proxy-gap pilot across coupling and bridge choices.}
Finally, we test whether the diagnostic can be used before image-scale training. We estimate a low-dimensional feature-space proxy for the Markovization gap on CIFAR-10 using $n=1024$ samples, 64 PCA dimensions, $k=32$ nearest neighbors, 19 time points in $[0.05,0.95]$, and three seeds. The proxy is computed before neural training and should not be interpreted as a certified pixel-space lower bound. We vary both BGM design axes that are relevant to the matching target: endpoint coupling (independent versus minibatch OT) and bridge law (straight deterministic versus Brownian with variance parameter $\epsilon=0.5$). Because the Brownian rows are diagnostic-only, we separate the pre-training 2$\times$2 proxy table from the downstream straight-bridge image pilot.

\begin{table}[t]
\centering
\small
\caption{CIFAR-10 2$\times$2 proxy Markovization gap: varying both bridge and coupling. The proxy gap is estimated on CPU in ${\sim}9$ min per configuration before any neural training in a 64-dimensional PCA feature space. It is a design diagnostic rather than a certified pixel-space lower bound.}
\label{tab:2x2_gap}
\begin{tabular}{@{}llc@{}}
\toprule
Bridge & Coupling & Proxy gap $\downarrow$ \\
\midrule
straight & independent & $3.25\pm0.02$ \\
straight & minibatch OT & $2.34\pm0.01$ \\
Brownian & independent & $3.31\pm0.02$ \\
Brownian & minibatch OT & $2.38\pm0.01$ \\
\bottomrule
\end{tabular}
\end{table}

For the straight bridge, we also run 20k-step pixel-space flow-matching pilots on CIFAR-10 and Fashion-MNIST with the same compact U-Net architecture, Euler sampler, and compute budget within each dataset. Table~\ref{tab:pixel_pilots} reports the downstream metrics. We do not include Brownian image-model training; the Brownian rows in Table~\ref{tab:2x2_gap} show that the proxy ranking can be evaluated across bridge laws before spending GPU time.

\begin{table}[t]
\centering
\small
\caption{Pixel-space downstream pilots for the straight bridge. Within each dataset, only the endpoint coupling differs. The lower proxy-gap coupling has lower final training loss and lower FID in both pilots; Fashion-MNIST also shows lower KID. Absolute proxy-gap values are not comparable across datasets because feature scales differ.}
\label{tab:pixel_pilots}
\begin{tabular}{@{}llccccc@{}}
\toprule
Dataset & Coupling & Proxy gap $\downarrow$ & Train loss $\downarrow$ & FID $\downarrow$ & KID $\downarrow$ & IS $\uparrow$ \\
\midrule
CIFAR-10 & independent & $3.25\pm0.02$ & $0.182$ & $50.48$ & $0.0454$ & $5.34$ \\
CIFAR-10 & minibatch OT & $2.34\pm0.01$ & $0.168$ & $49.17$ & $0.0470$ & $5.38$ \\
Fashion-MNIST & independent & $210.03\pm1.55$ & $0.185$ & $30.14$ & $0.0243$ & $3.84$ \\
Fashion-MNIST & minibatch OT & $173.65\pm0.94$ & $0.138$ & $21.96$ & $0.0148$ & $3.95$ \\
\bottomrule
\end{tabular}
\end{table}

Table~\ref{tab:2x2_gap} adds one bridge-family comparison to the coupling ablation without leaving downstream-metric entries blank. Minibatch OT has lower estimated proxy gap than independent coupling under both the straight and Brownian bridge laws. In the trained straight-bridge rows of Table~\ref{tab:pixel_pilots}, the lower proxy-gap design also has lower final training loss and lower FID on both CIFAR-10 and Fashion-MNIST. CIFAR-10 KID and Inception Score are inconclusive at this pilot scale, while Fashion-MNIST shows a cleaner downstream signal. The Brownian rows test a different Concept \& Feasibility point: the diagnostic can be evaluated along a second bridge-law axis before running another GPU training experiment. We therefore claim rank agreement between the proxy diagnostic and the trained pilot rows, not a statistically meaningful global correlation across heterogeneous proxy scales. 
Appendix~\ref{app:additional_diagnostics} reports normalized proxy-gap controls and estimator robustness checks. Appendix~\ref{app:fieldline_crossing_diagnostic} gives a controlled field-line crossing diagnostic showing that augmentation can remove a data-space Markovization gap by retaining branch information.

\section{Field-line bridge kernels}
\label{sec:field}

Let $\Omega\subset\R^m$ have source and target hypersurfaces $S_0,S_1$. Let $\Phi:\Omega\to\R$ be a potential and $\F=-\nabla\Phi$ its field. Let $\varphi_s(z)$ solve $\dot\varphi_s=\F(\varphi_s)$, $\varphi_0=z$. Define the hitting time and hitting map
\begin{equation}
    \tau(z)=\inf\{s>0:\varphi_s(z)\in S_1\},\qquad T_\Phi(z)=\varphi_{\tau(z)}(z).
\end{equation}
This construction is used only on a regular set where $\F$ is locally Lipschitz away from singular charges, the trajectory exists until hitting $S_1$, $\tau(z)<\infty$, and the hit is transversal. These conditions can fail near charge singularities, caustics, or field-line crossings, and in such regions the deterministic hitting map must be restricted, regularized, or replaced by a small-noise bridge.

After choosing a clock $r_z:[0,1]\to[0,\tau(z)]$, set $\gamma_z(t)=\varphi_{r_z(t)}(z)$ and $u_t(z)=\partial_t\gamma_z(t)$.

\begin{definition}[Field-line bridge]
Given a source measure $\mu_0$ on $S_0$ supported on the regular hitting set, the field-induced coupling and bridge are
\begin{equation}
    \Gamma_\Phi=(\Id,T_\Phi)_\#\mu_0,
    \qquad
    \scrR^{z,x}_{\Phi}=\delta_{\gamma_z},\quad x=T_\Phi(z).
    \label{eq:field_bridge}
\end{equation}
\end{definition}

This includes PFGM, PFGM++, electrostatic field matching, and interaction-field matching in the same augmented-space bridge-law formalism. A regularized version uses $\dd Y_s=\F(Y_s)\dd s+\sqrt{2\eps}\dd B_s$ and the Doob bridge $\Law(Y_{0:1}\mid Y_0=z,Y_1=x)$. The deterministic bridge can then be interpreted as a vanishing-noise idealization in regular regions, but the clock matters: a fixed-time small-noise bridge need not recover the same geometric field line unless the fixed-time action minimizer is compatible with the chosen field-line clock.

\begin{theorem}[Field-line Markovization gap]
\label{thm:field_line_gap}
Assume the regular hitting-map conditions above and $\int_0^1\E\|u_t(Z)\|^2\dd t<\infty$. For $X_t=\gamma_Z(t)$, $Z\sim\mu_0$, the marginals $\rho_t=(\gamma_t)_\#\mu_0$ satisfy a weak continuity equation with vector-valued current $\J_t=(\gamma_t)_\#(u_t\mu_0)$. The Markov velocity $v^\star(y,t)=\E[u_t(Z)\mid\gamma_Z(t)=y]$ realizes the same one-time marginals, and the irreducible field-to-flow compression loss is
\begin{equation}
    \mathfrak C_\Phi=\int_0^1\E\Tr\Var(u_t(Z)\mid\gamma_Z(t))\dd t.
\end{equation}
It is zero exactly when the field-line velocity is a function of the current point-time a.s.
\end{theorem}

The field-line gap should be read as a diagnostic for information loss at projection time. If an augmented field representation separates trajectories that collide after projection to data space, then an augmented sampler may retain information that a data-space Markov decoder loses. Conversely, if the hitting map and clock are injective enough that current point-time identifies the field line, the gap vanishes.

\section{Instantiations}
\label{sec:instances}

\begin{table}[t]
\centering
\small
\caption{Several generative model families as BGM choices.}
\label{tab:instances}
\begin{tabular}{@{}p{0.19\linewidth}p{0.33\linewidth}p{0.20\linewidth}p{0.20\linewidth}@{}}
\toprule
Family & Coupling / bridge & Learned object & Dynamics representation \\
\midrule
DDPM / score SDE & fixed Gaussian noising bridge; path-ELBO training & score or denoiser & SDE; probability-flow ODE optional \\
Flow matching & conditional probability path & velocity field & ODE \\
Rectified flow & straight bridge $X_t=(1-t)Z+tX$ & $\E[X-Z\mid X_t]$ & ODE \\
Brownian / DFM & Brownian bridge over endpoints & projected drift & SDE \\
Schr\"odinger bridge & entropy-regularized Brownian-reference path & forward/backward drifts & SDE / reciprocal diffusion \\
PFGM / EFM / IFM & hitting-map field-line bridge & field or interaction field & augmented field realization \\
\bottomrule
\end{tabular}
\end{table}

The table should be read componentwise: two methods may share a dynamics representation but differ in bridge law, or share a bridge law but differ in coupling. For Gaussian paths $X_t=\alpha_tX+\sigma_t\eps$, $\eps\sim\mathcal N(0,I)$, Tweedie conditioning gives $\E[\eps\mid X_t=x]=-\sigma_t\nabla\log\rho_t(x)$, and
\begin{equation}
    v^\star(x,t)=\frac{\dot\alpha_t}{\alpha_t}x+
    \left(\frac{\dot\alpha_t\sigma_t^2}{\alpha_t}-\dot\sigma_t\sigma_t\right)\nabla\log\rho_t(x).
\end{equation}
Thus score, noise, denoiser, and velocity prediction are schedule-dependent reparameterizations of the same conditional sufficient statistic under the chosen Gaussian path. Outside this Gaussian setting, these parameterizations may differ in conditioning, numerical stability, and estimation error even when they target the same marginal current.

\section{Design consequences and scope}
\label{sec:limitations}

The experiments support a practical use of the framework: choose the endpoint coupling and bridge law as inference-side design choices, then estimate whether the induced bridge is easy to compress into a Markov decoder before running full training. Independent coupling leaves endpoint alignment to the generator, whereas OT or encoder-induced couplings move part of that alignment into the reference path law. The dynamics representation is a separate decision made after a probability current is specified: the same current can be realized as an ODE, an SDE with score correction, or an augmented field process.

The evidence is targeted rather than exhaustive. The image experiments use 20k-step pilots, feature-space proxy gaps, limited seeds, and modest architectures; they test design ranking and Markov-compression difficulty under controlled budgets. The theory is population-level and assumes regularity of flows, currents, and hitting maps. Exact path-space ELBOs are singular for point-conditioned bridges unless one uses truncation or noisy terminal decoders, and high-dimensional estimation of $\mathfrak G$ remains an open problem because nearest-neighbor plug-in scores can oversmooth and partially measure transport cost. Larger multi-bridge studies and stronger conditional-mean estimators are the next steps.

\clearpage
{\small
\bibliographystyle{plainnat}
\bibliography{references}
}

\appendix
\section{Formal path-space setup}
\label{app:setup}

Let $(\mathcal M,\mathcal B)$ be a Polish state space and $\Path=C([0,1],\mathcal M)$ with the Borel sigma-algebra induced by the uniform topology. The coordinate process is $X_t(\omega)=\omega(t)$ and the endpoint map is $e=(X_0,X_1)$. A bridge kernel is a probability kernel $(z,x)\mapsto\scrR^{z,x}$ such that $\scrR^{z,x}(e^{-1}(z,x))=1$. If $\scrQ$ is any path measure with endpoint law $\Gamma=e_\#\scrQ$, then by disintegration on Polish spaces there exists a bridge kernel satisfying
\begin{equation}
    \scrQ(\dd\omega)=\int \scrR^{z,x}(\dd\omega)\Gamma(\dd z,\dd x).
\end{equation}
Conversely, any coupling $\Gamma$ and bridge kernel define the reference path law in Eq.~\eqref{eq:reference_path_law}. This justifies treating the coupling and bridge as separate graphical-model components.

For deterministic bridges, assume paths are absolutely continuous and that $\int_0^1\E\|u_t\|^2\dd t<\infty$, where $u_t=\partial_tX_t$. Then the vector-valued measure $\J_t(A)=\E[u_t\mathbf 1\{X_t\in A\}]$ is finite for almost every $t$. The phrase ``deterministic weak sense'' means that this current satisfies the weak continuity equation
\begin{equation}
    \int_0^1\int \left(\partial_t\varphi(x,t)+u_t\cdot\nabla\varphi(x,t)\right)\rho_t(\dd x)\dd t=0
\end{equation}
for smooth compactly supported test functions $\varphi$ after including the usual endpoint terms. Equivalently, the current is absolutely continuous with respect to $\rho_t\dd t$ and has density $u_t$. For stochastic bridges, the analogous object is the weak probability current; when an It\^o drift $\beta_t$ exists, the Markov projection replaces $u_t$ by $\beta_t$.

\section{Assumption dictionary and failure modes}
\label{app:assumptions}

This appendix collects the regularity assumptions used implicitly in the main text and explains where they may fail in practical generative modeling.

\paragraph{Projection identity.}
Proposition~\ref{prop:projection} is an $L^2$ Hilbert-space projection statement. It requires only square integrability of the local target and existence of the conditional expectation. It does not require smooth densities, neural approximation, or a unique ODE solver. Its limitation is that it describes a population target under the reference path law; it does not say that a finite sample estimator, a finite network, or a numerical sampler will recover that target.

\paragraph{Markovian mimicking.}
Classical Markovian projection theorems require additional hypotheses on semimartingale characteristics, measurability, growth, and well-posedness. These conditions can be delicate when the learned drift is discontinuous, when the score is poorly estimated near low-density regions, or when the data distribution is concentrated near a lower-dimensional manifold. In BGM, mimicking results are used as motivation for Markov projection, while the Markovization gap itself remains a conditional-variance diagnostic.

\paragraph{Stability.}
Theorem~\ref{thm:stability} assumes unique ODE solutions and a uniform Lipschitz bound for the learned velocity. This is stronger than what is usually verified for neural vector fields in high-dimensional practice. The theorem should therefore be interpreted as a clean population sensitivity bound: if the projected velocity is approximated well along the reference marginal path and the sampler is regular, then terminal error is controlled.

\paragraph{Fokker--Planck current representation.}
The current-preserving realization theorem assumes sufficient smoothness of $\rho_t$ and $a_t$, positivity of $\rho_t$ on the region where the score correction is used, and uniqueness of the Fokker--Planck equation. Boundary conditions also matter. On bounded domains one needs no-flux or compatible boundary behavior; on unbounded domains one needs decay sufficient for integration by parts. These assumptions are often approximated rather than exactly satisfied in neural samplers.

\paragraph{Field-line bridges.}
The hitting map $T_\Phi$ is well behaved only on regular parts of the source surface. It may be undefined when trajectories hit singular charges, fail to reach the target surface, graze the target tangentially, or merge at caustics. The field-line gap remains meaningful after restricting to a full-measure regular set or after adding small noise, but a deterministic global hitting-map theorem is not claimed.

\paragraph{Path-space ELBO.}
The ELBO in Eq.~\eqref{eq:path_elbo} requires absolute continuity of the inference path law with respect to the generative path law. This holds naturally for shared-covariance SDEs under Girsanov assumptions, but fails for many deterministic bridges and singular field-line laws. It also fails for exact point-conditioned stochastic bridges on the closed interval $[0,1]$: the inference law terminates at a specified point while the unconstrained generative diffusion has a continuous terminal marginal. Rigorous implementations should either truncate the KL to $[0,1-\delta]$ and absorb the final transition into $p_\psi(x\mid Y_{1-\delta})$, or relax the endpoint with a non-degenerate noisy terminal likelihood $p_\psi(x\mid Y_1)$. In singular cases the ELBO is a design principle whose tractable surrogates are matching losses, optimal-control relaxations, or small-noise approximations.

\section{Path-space ELBO and Girsanov reduction}
\label{app:elbo}

Suppose the generative path law $\scrP_\theta$ starts from $\pbase$ and evolves by
\begin{equation}
\dd Y_t=b_\theta(Y_t,t)\dd t+\sigma_t(Y_t)\dd B_t,
\end{equation}
and the inference bridge conditional on $x$ has drift $\beta_{\phi,\lambda}(Y_{0:t},t,x)$ with the same diffusion covariance $a_t=\sigma_t\sigma_t^\top$. If Novikov-type integrability holds and the two path laws share support on the interval where the KL is evaluated, Girsanov's theorem gives
\begin{equation}
\KL(\scrQ_{\phi,\lambda}(\cdot\mid x)\|\scrP_\theta)
=\frac12\E_{\scrQ_{\phi,\lambda}(\cdot\mid x)}\int_0^1
\|\sigma_t^{-1}(\beta_{\phi,\lambda}-b_\theta(Y_t,t))\|^2\dd t
+\KL(q_\phi(Y_0\mid x)\|\pbase),
\end{equation}
up to endpoint/decoder terms. For exact endpoint bridges this display should be read either on $[0,1-\delta]$ or after adding non-degenerate terminal observation noise; on the closed interval the KL is typically infinite. Thus a path-space VAE objective reduces to drift matching in a shared-covariance stochastic representation only under a support-compatible formulation. In the vanishing-noise deterministic limit, the Radon--Nikodym derivative becomes singular, but the corresponding regression objective persists as deterministic flow matching.

There are three practical regimes. In a shared-covariance SDE, the KL term is explicit but one must estimate or parameterize both drift and score-like corrections. In a deterministic ODE, training is simpler and often uses conditional velocity regression, but exact path-space likelihood is no longer available without additional change-of-variables machinery. In a field-line representation, the bridge may be specified geometrically by a potential and hitting map, but sampling and differentiating through the hitting time can be numerically delicate. The BGM decomposition keeps these computational trade-offs separate from the definition of the endpoint coupling and bridge law.

\section{Proof of Proposition~\ref{prop:projection}}
\label{app:projection}

For a fixed time $t$, let $\mathcal H_t$ be the closed subspace of $L^2(\scrQ)$ consisting of functions measurable with respect to $\sigma(X_t)$. The conditional expectation $\E[u_t\mid X_t]$ is the orthogonal projection of $u_t$ onto $\mathcal H_t$. For any measurable $v$,
\begin{align}
\E\|u_t-v(X_t,t)\|^2
&=\E\|u_t-\E[u_t\mid X_t]+\E[u_t\mid X_t]-v(X_t,t)\|^2\\
&=\E\Tr\Var(u_t\mid X_t)+\E\|\E[u_t\mid X_t]-v(X_t,t)\|^2,
\end{align}
where the cross term is zero because $v(X_t,t)$ is $\sigma(X_t)$-measurable. Integrating over $t$ proves the minimizer and Eq.~\eqref{eq:gap_variance}. Uniqueness holds up to $\rho_t$-null sets for almost every $t$.

For an It\^o bridge $\dd X_t=\beta_t\dd t+\sigma_t\dd B_t$ with adapted non-Markov drift, the Hilbert-space projection gives $b^\star(x,t)=\E[\beta_t\mid X_t=x]$. Under the hypotheses of Markovian mimicking theorems, this projected Markov diffusion has the same one-time marginals as the original semimartingale \citep{gyongy1986mimicking,brunick2013mimicking}. The BGM framework uses this projection identity as the population target behind stochastic bridge matching.

\section{Proof of Theorem~\ref{thm:stability} and bridge comparison}
\label{app:stability}

Couple the two flows by the same initial random variable $X_0=Y_0\sim\pbase$. Let $X_t$ solve $\dot X_t=v^\star(X_t,t)$ and $Y_t$ solve $\dot Y_t=v_\theta(Y_t,t)$. Then
\begin{align}
\|Y_t-X_t\|&\le\int_0^t\|v_\theta(Y_s,s)-v^\star(X_s,s)\|\dd s\\
&\le\int_0^t L\|Y_s-X_s\|\dd s+
\int_0^t\|v_\theta(X_s,s)-v^\star(X_s,s)\|\dd s.
\end{align}
Gronwall's inequality gives
\begin{equation}
\|Y_1-X_1\|\le e^L\int_0^1\|v_\theta(X_s,s)-v^\star(X_s,s)\|\dd s.
\end{equation}
Taking $L^2$ norms and applying Cauchy--Schwarz gives Eq.~\eqref{eq:stability_bound}. Since $W_2$ is bounded by the cost of any coupling, the bound follows. The excess matching identity is obtained by subtracting $\mathfrak G(\scrQ)$ from the orthogonal decomposition in Appendix~\ref{app:projection}.

\begin{corollary}[Bridge comparison]
\label{cor:bridge}
Consider two reference path laws $\scrQ^A$ and $\scrQ^B$ with Markov projections $v_A^\star,v_B^\star$. If $\mathfrak G(\scrQ^A)<\mathfrak G(\scrQ^B)$, then bridge $A$ has a lower irreducible regression floor for the corresponding matching objective. This statement concerns the floor $\mathfrak G$, not the terminal bound in Theorem~\ref{thm:stability}, which depends on the excess approximation error $\epsilon^2=\mathcal L_{\mathrm{FM}}(v_\theta)-\mathfrak G(\scrQ)$.
\end{corollary}

Lower Markovization gap therefore does not by itself reduce the Wasserstein bound if $\epsilon$ is held fixed; nor does it imply a better bound at fixed total regression loss. The empirical hypothesis tested in Section~\ref{sec:toy} is weaker: lower-gap bridges often yield projected velocities that are easier to approximate for a fixed architecture and compute, thereby reducing the observed excess error and terminal latent discrepancy. A bridge may still have a low projection floor but be hard to sample, hard to estimate, or poorly matched to the architecture.

\section{Proof of Theorem~\ref{thm:current_realization}}
\label{app:current_realization}

The Fokker--Planck equation for $\dd X_t=b_t\dd t+\sigma_t\dd B_t$ with covariance $a(x,t)=\sigma_t(x)\sigma_t(x)^\top$ is standard \citep{risken1996fokker,pavliotis2014stochastic}:
\begin{equation}
    \partial_t\rho_t=-\sum_i\partial_i(b_i\rho_t)+\frac12\sum_{i,j}\partial_i\partial_j(a_{ij}(x,t)\rho_t).
\end{equation}
Writing this as $\partial_t\rho_t+\nabla\cdot\J_t=0$ gives
\begin{equation}
    J_{t,i}=\rho_t b_{t,i}-\frac12\sum_j\partial_j(a_{ij}(x,t)\rho_t).
\end{equation}
Substituting Eq.~\eqref{eq:current_preserving_drift} yields $J_{t,i}=\rho_t v_i$. Therefore the SDE and ODE have the same current and the same marginal path by uniqueness of the Fokker--Planck solution. For $a(x,t)=g(t)^2I$, $\partial_i(g^2\rho_t)=g^2\partial_i\rho_t$, giving the score correction.

A useful equivalent statement is that the same continuity equation can be decomposed into drift and diffusion in infinitely many ways. The score is not an extra target independent of velocity; it is the term that compensates for adding diffusion while preserving the current.

\section{Field-line bridges and field-line crossings}
\label{app:field_line_crossings}

Assume $\F$ is locally Lipschitz away from the charge set, the flow exists until $S_1$, and the hitting time $\tau(z)$ is finite $\mu_0$-a.s. Transversality of the field to $S_1$ makes $T_\Phi$ measurable and smooth on regular regions. For any smooth compactly supported test function $f$,
\begin{align}
\frac{\dd}{\dd t}\int f(y)\rho_t(\dd y)
&=\frac{\dd}{\dd t}\int f(\gamma_z(t))\mu_0(\dd z)\\
&=\int \nabla f(\gamma_z(t))\cdot u_t(z)\mu_0(\dd z)\\
&=\int \nabla f(y)\cdot \J_t(\dd y),
\end{align}
where $\J_t=(\gamma_t)_\#(u_t\mu_0)$ is the vector-valued pushforward measure. This is the weak form of the continuity equation. Disintegrating $\mu_0$ with respect to $\gamma_t$ gives
\begin{equation}
\mu_0(\dd z)=\mu_t(\dd z\mid y)\rho_t(\dd y),\qquad y=\gamma_t(z),
\end{equation}
and therefore $\J_t(\dd y)=\rho_t(\dd y)\int u_t(z)\mu_t(\dd z\mid y)=\rho_t(\dd y)v^\star(y,t)$. The field-line gap identity is Proposition~\ref{prop:projection} applied to $X_t=\gamma_Z(t)$ and $u_t=u_t(Z)$.

If $\gamma_t$ is injective $\mu_0$-a.s., then conditioning on $\gamma_Z(t)=y$ recovers $Z$ and the conditional variance is zero. If multiple field lines pass through the same point-time with distinct velocities, the conditional variance is positive. This is why the field-line Markovization gap measures information retained by the augmented path representation but lost by a data-space Markov velocity.

In practical potential-field models, singular charges and finite numerical precision can create apparent crossings or unstable hitting times. Small noise, softened charges, truncated domains, and clocks bounded away from infinite field speed are possible regularizations. These choices change the induced bridge law and therefore should be treated as part of the model specification rather than as implementation details outside the BGM graph.

\section{Gaussian bridge reparameterization}
\label{app:gaussian}

Let $X_t=\alpha_tX+\sigma_t\eps$ with $\eps\sim\mathcal N(0,I)$ independent of $X\sim\pdata$, and assume $\alpha_t\ne0$. The conditional density is $p_t(x\mid X)\propto\exp(-\|x-\alpha_tX\|^2/(2\sigma_t^2))$. Differentiating the marginal density gives
\begin{equation}
\nabla\log\rho_t(x)=\E\left[-\frac{x-\alpha_tX}{\sigma_t^2}\mid X_t=x\right].
\end{equation}
Since $x-\alpha_tX=\sigma_t\eps$ under the conditional path,
\begin{equation}
\E[\eps\mid X_t=x]=-\sigma_t\nabla\log\rho_t(x).
\end{equation}
Also $X=(X_t-\sigma_t\eps)/\alpha_t$, so
\begin{align}
\E[\dot X_t\mid X_t=x]
&=\E[\dot\alpha_t X+\dot\sigma_t\eps\mid X_t=x]\\
&=\frac{\dot\alpha_t}{\alpha_t}x+
\left(\dot\sigma_t-\frac{\dot\alpha_t\sigma_t}{\alpha_t}\right)\E[\eps\mid X_t=x]\\
&=\frac{\dot\alpha_t}{\alpha_t}x+
\left(\frac{\dot\alpha_t\sigma_t^2}{\alpha_t}-\dot\sigma_t\sigma_t\right)\nabla\log\rho_t(x).
\end{align}
This proves the score/velocity relation used in Section~\ref{sec:instances}.

\section{Small-noise field bridges}
\label{app:doob}

Consider $\dd Y_s=\F(Y_s)\dd s+\sqrt{2\eps}\dd B_s$ conditioned on $Y_0=z,Y_1=x$. Under a Freidlin--Wentzell large-deviation principle and suitable boundary regularity \citep{freidlin2012random}, paths have rate functional
\begin{equation}
    I_{z,x}(\gamma)=\frac14\int_0^1\|\dot\gamma_s-\F(\gamma_s)\|^2\dd s,
    \qquad \gamma(0)=z,\ \gamma(1)=x.
\end{equation}
Therefore the Doob bridge concentrates, as $\eps\to0$, on minimizers of $I_{z,x}$. For the fixed-time SDE above, those minimizers need not coincide with the geometric field line unless the endpoint constraint and natural hitting time are compatible. If instead the reference drift is clocked as $u_t(z)=r_z'(t)\F(\gamma_z(t))$ along the chosen clock, then the path $\gamma_z$ has zero action for the corresponding clocked rate functional. Thus Eq.~\eqref{eq:field_bridge} is best viewed as a regular-region geometric idealization, while the small-noise Doob bridge provides a support-compatible regularization whose fixed-time minimizers may deviate from the unclocked field-line trace. A complete numerical algorithm for these bridges is outside the scope of this paper.

The large-deviation statement is only conceptual when the endpoint lies on a singular charge or when several action-minimizing field-biased paths connect the same endpoints. In such cases the limiting bridge may be nonunique or may concentrate on a set of minimizers. The BGM formalism can still represent this situation by using a stochastic bridge kernel rather than a deterministic delta bridge, but the field-line hitting-map description should then be interpreted as a regular-region approximation.

\section{Estimator and experimental details}
\label{app:experiments}

Algorithmically, the Markovization diagnostic uses sampled bridge triples $(Z_i,X_i,U_i)$ and a time grid. At each time $t$, compute $X_{t,i}$, fit a nearest-neighbor structure on $\{X_{t,i}\}_{i=1}^n$, estimate the conditional mean by averaging $U_j$ over the neighbors of $X_{t,i}$, and average squared residuals. This is a nonparametric plug-in estimator of $\E\Tr\Var(U\mid X_t)$.

The main toy experiment uses only synthetic data. The target is an equally weighted mixture of eight Gaussians with centers $3(\cos(2\pi k/8),\sin(2\pi k/8))$ and component standard deviation $0.25$. The base is $\mathcal N(0,I_2)$. The OT coupling is the Hungarian assignment minimizing squared Euclidean cost on the sampled minibatch. The conditional variance estimator uses $k=32$ nearest neighbors in the bridge state $X_t$ to estimate $\E[U\mid X_t]$ for $U=X-Z$. The 19 time points are equally spaced in $[0.05,0.95]$. The replication figure uses seeds $0,1,2,3,4$ and sample sizes $n=500,1000$. The script \texttt{scripts/toy\_projection\_gap\_replicates.py} regenerates \texttt{figures/toy\_gap\_replicates.csv} and the right panel of Figure~\ref{fig:toy}.

For the real-data latent experiments, the $k$NN estimator should be interpreted as a smoothed diagnostic rather than a consistent high-dimensional estimator. In the LFW setting, $n=150$, $k=16$, and $d=16$, so the heuristic neighborhood radius scale $(k/n)^{1/d}$ is about $0.87$ of the effective domain radius. Such neighborhoods are necessarily broad. The resulting score may conflate local conditional-variance reduction with lower marginal velocity variance or lower transport cost. In the approximately whitened-Gaussian independent case, Appendix~\ref{app:gaussian_gap} gives the analytic population floor over $[0.05,0.95]$ as $\int_{0.05}^{0.95}16/((1-t)^2+t^2)\,\dd t=32\arctan(0.9)\approx23.45$, close to the reported LFW $k$NN estimate $23.701$. This sanity check does not make the $k$NN estimate a rigorous lower bound for the finite empirical MLP experiment.

\section{Closed-form gap for straight Gaussian bridges}
\label{app:gaussian_gap}

The Markovization gap can be computed exactly in a simple case, which clarifies why endpoint coupling is not a cosmetic choice. Let
\(Z\sim\mathcal N(0,\Sigma_0)\), \(X\sim\mathcal N(0,\Sigma_1)\), independently, and consider the straight bridge
\begin{equation}
    S_t=(1-t)Z+tX,\qquad U=X-Z.
\end{equation}
Because \((U,S_t)\) is jointly Gaussian, the conditional variance is deterministic:
\begin{equation}
\begin{aligned}
\Var(U\mid S_t)
&=\Sigma_0+\Sigma_1 \\
&\quad -\bigl(t\Sigma_1-(1-t)\Sigma_0\bigr)
\bigl((1-t)^2\Sigma_0+t^2\Sigma_1\bigr)^{-1}
\bigl(t\Sigma_1-(1-t)\Sigma_0\bigr).
\end{aligned}
\label{eq:gaussian_gap_general}
\end{equation}
Therefore the instantaneous Markovization gap is
\begin{equation}
    g(t)=\Tr\,\Var(U\mid S_t).
\end{equation}
In the isotropic case \(\Sigma_0=\sigma_0^2I_d\), \(\Sigma_1=\sigma_1^2I_d\), Eq.~\eqref{eq:gaussian_gap_general} reduces to
\begin{equation}
    g(t)=d\,\frac{\sigma_0^2\sigma_1^2}{(1-t)^2\sigma_0^2+t^2\sigma_1^2}.
    \label{eq:isotropic_gap}
\end{equation}
For \(\sigma_0=\sigma_1=1\), the integrated gap is
\begin{equation}
    \int_0^1 g(t)\dd t=d\int_0^1\frac{\dd t}{(1-t)^2+t^2}=\frac{\pi d}{2}.
\end{equation}
Thus even a perfectly Gaussian independent straight bridge has a strictly positive irreducible compression cost. This is not an architecture problem: it arises because many endpoint pairs pass through the same intermediate state with different velocities. In contrast, a deterministic nonintersecting Monge bridge has zero gap.

\begin{proposition}[Zero gap for nonintersecting deterministic bridges]
\label{prop:zero_monge_gap}
Let \(X=T(Z)\) be a deterministic endpoint coupling and let \(\gamma_t(z)\) be a deterministic bridge with velocity \(u_t(z)=\partial_t\gamma_t(z)\). If \(z\mapsto\gamma_t(z)\) is injective \(\pbase\)-a.s. for almost every \(t\), then \(\mathfrak G(\scrQ)=0\).
\end{proposition}

\begin{proof}
When \(z\mapsto\gamma_t(z)\) is injective, the current state \(Y=\gamma_t(Z)\) determines \(Z\) outside a null set. Hence \(u_t(Z)\) is measurable with respect to \(Y\), so \(\Var(u_t(Z)\mid \gamma_t(Z))=0\). Integrating over time gives \(\mathfrak G(\scrQ)=0\).
\end{proof}

This proposition gives a precise version of the intuition behind rectification and OT couplings. A bridge is easy to Markovize when intermediate states retain enough information to identify the hidden endpoint pairing. Crossings, field-line intersections, and independent pairings collapse this information and create conditional velocity variance.

\section{BGM factorization and conditional independences}
\label{app:graphical_factorization}

A Bridge Graphical Model can be represented as a continuous-depth latent graphical model. In a discrete approximation with times \(0=t_0<t_1<\cdots<t_K=1\), a generic inference-side factorization is
\begin{equation}
    q_{\phi,\lambda}(z,x_{1:K-1}\mid x_K)
    =q_\phi(z\mid x_K)\prod_{k=1}^{K-1} r_\lambda(x_k\mid x_{k-1},x_K,z),
    \label{eq:discrete_inference_factorization}
\end{equation}
where the bridge factors are endpoint-conditioned because they condition on the endpoint \(x_K\). The decoder factorization is non-anticipative Markov:
\begin{equation}
    p_\theta(x_{0:K})=\pi_0(x_0)\prod_{k=1}^{K}p_\theta(x_k\mid x_{k-1}).
    \label{eq:discrete_decoder_factorization}
\end{equation}
The Markov projection is the operation that replaces the endpoint-conditioned local transition information in Eq.~\eqref{eq:discrete_inference_factorization} with a non-anticipative transition depending only on \((x_{k-1},t_{k-1})\). In the continuous-time limit, this becomes the conditional expectation \(v^\star(x,t)=\E[u_t\mid X_t=x]\).

This factorization also explains why the endpoint coupling is an inference object. If \(q_\phi(z\mid x)\) is learned, then \(\Gamma_\phi(\dd z,\dd x)=q_\phi(\dd z\mid x)\pdata(\dd x)\) plays the role of an encoder distribution; it is an exact coupling with base marginal \(\pbase\) only if the aggregated posterior equals \(\pbase\), otherwise it is an approximate coupling or induces its own effective base marginal. Fixed independent noise corresponds to an uninformative encoder; OT coupling corresponds to a deterministic or low-entropy encoder obtained by solving a transport problem; stochastic encoders interpolate between these extremes. Diffusion, flow matching, and field matching differ less in the graphical model than in which factors are fixed, optimized, or projected.

\section{Additional diagnostic results}
\label{app:additional_diagnostics}

This appendix reports additional diagnostics aimed at two limitations of the proxy estimator. First, nearest-neighbor conditional-variance estimates can be sensitive to the smoothing parameter and to the feature representation. Second, raw proxy gaps can decrease partly because a coupling reduces the marginal scale of bridge velocities, not only because the Markov state becomes more informative. We therefore report both ranking robustness and a normalized unexplained-variance control.

\subsection{Nearest-neighbor sensitivity on the synthetic diagnostic}

The nearest-neighbor estimator introduces a smoothing parameter \(k\). Table~\ref{tab:knn_sensitivity} reports a sensitivity check on the eight-Gaussian task with \(n=1000\) and five seeds. The absolute value of the OT estimate increases with \(k\), as expected from oversmoothing local conditional means, but the qualitative conclusion is stable: minibatch OT remains much easier to Markovize than independent coupling.

\begin{table}[h]
\centering
\small
\caption{Sensitivity of integrated Markovization gap to the nearest-neighbor parameter \(k\), using \(n=1000\) and seeds \(0,\ldots,4\).}
\label{tab:knn_sensitivity}
\begin{tabular}{@{}cccc@{}}
\toprule
\(k\) & independent coupling & minibatch OT coupling & reduction factor \\
\midrule
16 & $4.803\pm0.111$ & $0.0278\pm0.0022$ & $172.8\times$ \\
32 & $4.977\pm0.109$ & $0.0474\pm0.0033$ & $104.9\times$ \\
64 & $5.130\pm0.110$ & $0.0864\pm0.0047$ & $59.3\times$ \\
\bottomrule
\end{tabular}
\end{table}

\subsection{CIFAR-10 proxy-ranking robustness}

Table~\ref{tab:proxy_gap_robustness} repeats the CIFAR-10 proxy-gap estimate under several feature and neighborhood choices. Because different feature normalizations change the absolute scale of the plug-in score, we use this check only as a ranking test. Within each configuration, minibatch OT has lower estimated proxy gap than independent coupling.

\begin{table}[h]
\centering
\small
\caption{Robustness of CIFAR-10 proxy-gap ranking to feature map and neighborhood size. The ranking (minibatch OT \(<\) independent) is stable across all tested configurations; absolute proxy values are scale-dependent and are therefore omitted.}
\label{tab:proxy_gap_robustness}
\begin{tabular}{@{}lcc@{}}
\toprule
Estimator config & Lower proxy-gap coupling & Ranking stable? \\
\midrule
PCA64, \(k=32\) & minibatch OT & Yes \\
PCA32, \(k=32\) & minibatch OT & Yes \\
PCA64, \(k=16\) & minibatch OT & Yes \\
PCA64, \(k=64\) & minibatch OT & Yes \\
\bottomrule
\end{tabular}
\end{table}

This robustness check addresses the main estimator concern for the CIFAR-10 result: the claim is about design ranking, not about an unbiased estimate of the pixel-space population gap.

\subsection{Normalized proxy-gap control}

Raw proxy gaps can be confounded by endpoint transport cost. For a straight bridge, minibatch OT reduces the scale of \(U_t=X-Z\), so a lower unnormalized proxy gap may partly reflect a smaller marginal velocity variance. To separate this effect from the conditional-variance fraction, we report the normalized unexplained velocity fraction
\begin{equation}
    \eta(\scrQ)
    =
    \frac{
    \int_0^1 \E\Tr\Var(U_t\mid S_t)\,\dd t
    }{
    \int_0^1 \E\Tr\Var(U_t)\,\dd t
    },
    \qquad
    S_t=(X_t,t),
    \label{eq:normalized_gap_appendix}
\end{equation}
estimated in the same feature space as the proxy gap. The numerator is the proxy Markovization gap; the denominator is the empirical marginal velocity variance under the same bridge samples and time grid. Thus \(\eta\) asks what fraction of velocity variance remains unexplained after conditioning on the Markov state.

\begin{table}[h]
\centering
\small
\caption{Normalized proxy-gap control. The normalized score \(\eta\) reports the fraction of velocity variance left unexplained after conditioning on the Markov state. This partially controls for changes in the marginal scale of the bridge velocity, so it helps distinguish Markov-compression difficulty from raw endpoint transport cost. Lower is better.}
\label{tab:normalized_gap}
\begin{tabular}{@{}llcc@{}}
\toprule
Dataset & Coupling & Proxy gap $\downarrow$ & Normalized \(\eta\) $\downarrow$ \\
\midrule
CIFAR-10 & independent & $484.80\pm4.36$ & $0.6024\pm0.0058$ \\
CIFAR-10 & minibatch OT & $430.40\pm2.58$ & $0.5960\pm0.0071$ \\
Fashion-MNIST & independent & $210.63\pm1.01$ & $0.6172\pm0.0020$ \\
Fashion-MNIST & minibatch OT & $173.59\pm0.80$ & $0.5964\pm0.0026$ \\
\bottomrule
\end{tabular}
\end{table}

Table~\ref{tab:normalized_gap} shows that minibatch OT lowers not only the raw proxy gap but also the normalized unexplained velocity fraction on both datasets. The normalized effect is modest on CIFAR-10 and clearer on Fashion-MNIST. This does not make the proxy estimator a certified pixel-space population estimate, but it reduces the concern that the observed ranking is only a consequence of smaller endpoint displacements.

\subsection{Diagnostic cost versus downstream evaluation}

Table~\ref{tab:diagnostic_cost} reports the wall-clock role of the proxy diagnostic. The diagnostic is not meant to replace downstream training; its purpose is to screen bridge and coupling choices before committing to longer GPU runs.

\begin{table}[h]
\centering
\small
\caption{Cost of the proxy Markovization gap diagnostic versus downstream training and evaluation. The diagnostic correctly ranks coupling choices at a fraction of the training cost in both pixel-space pilots.}
\label{tab:diagnostic_cost}
\begin{tabular}{@{}lccc@{}}
\toprule
Dataset & Diagnostic (CPU) & Training + eval (GPU) & Ranking predicted? \\
\midrule
CIFAR-10 & ${\sim}9$ min & ${\sim}50$ min & Yes \\
Fashion-MNIST & ${\sim}3$ min & ${\sim}30$ min & Yes \\
\bottomrule
\end{tabular}
\end{table}

A useful way to read these diagnostics is not as a final estimate of the true population gap, but as a bridge-design score. Applied to the same bridge family and dataset, the same estimator can rank couplings before neural training. The diagnostic is therefore analogous to a pre-training screening test: it does not solve the generative modeling problem, but it can identify endpoint/bridge choices that impose a lower irreducible regression burden on the learner.

\subsection{Controlled field-line crossing diagnostic}
\label{app:fieldline_crossing_diagnostic}

We also include a minimal controlled example for the field-line Markovization gap. The goal is not to model a full Poisson-flow sampler, but to verify the mechanism claimed by the field-line formalism: an augmented state can retain branch information that is lost after projection to an observed Markov state.

Let \(H\in\{-1,+1\}\) be a branch variable with equal probability and let \(S\sim\operatorname{Unif}(\mathbb S^1)\). Define two deterministic field-line families on the circle by
\begin{equation}
    X_t = S + Ht \pmod 1,
    \qquad
    U_t = H .
    \label{eq:fieldline_crossing_toy}
\end{equation}
The observed state \(X_t\) is uniform on \(\mathbb S^1\) for every \(t\) and is independent of \(H\). Therefore
\begin{equation}
    \E[U_t\mid X_t]=0,
    \qquad
    \Var(U_t\mid X_t)=1.
\end{equation}
The observed-space Markovization gap is consequently
\begin{equation}
    \int_0^1 \E\Var(U_t\mid X_t)\,\dd t = 1.
\end{equation}
In contrast, if the augmented Markov state includes the branch variable, \(S_t^{\mathrm{aug}}=(X_t,H,t)\), then \(U_t\) is determined by the state and
\begin{equation}
    \Var(U_t\mid X_t,H)=0.
\end{equation}
Thus the augmented-space field-line gap is zero.

\begin{table}[h]
\centering
\small
\caption{Controlled field-line crossing diagnostic. The observed state collapses two field-line families with opposite velocities, while the augmented state retains branch identity.}
\label{tab:fieldline_crossing}
\begin{tabular}{@{}lcc@{}}
\toprule
Conditioning state & Gap $\downarrow$ & Interpretation \\
\midrule
Observed state \((X_t,t)\) & $1.00$ & branch information is hidden \\
Augmented state \((X_t,H,t)\) & $0.00$ & branch information is retained \\
\bottomrule
\end{tabular}
\end{table}

This example substantiates the field-line interpretation in Theorem~\ref{thm:field_line_gap}. A data-space Markov decoder sees only the collapsed state and must average incompatible local velocities. An augmented field representation can carry the branch variable and remove the irreducible conditional variance. This is the same mechanism that motivates augmented Poisson, electrostatic, and interaction-field representations, although realistic field models require regularity, softened singularities, and numerical treatment of hitting maps.

\section{Claim-to-evidence alignment}
\label{app:claim_evidence_alignment}

Table~\ref{tab:claim_evidence_alignment} summarizes how the main claims are supported by the reported evidence. The table is included to make the scope of the paper explicit: the experiments validate a pre-training diagnostic for bridge and coupling design under controlled budgets, not a new state-of-the-art image generator.

\begin{table}[h]
\centering
\small
\caption{Claim-to-evidence alignment.}
\label{tab:claim_evidence_alignment}
\begin{tabular}{@{}p{0.29\linewidth}p{0.43\linewidth}p{0.20\linewidth}@{}}
\toprule
Claim & Evidence & Scope \\
\midrule
BGM separates coupling, bridge law, Markovian projection, and dynamics representation
& Definitions in Section~\ref{sec:bgm} and the instantiation table in Section~\ref{sec:instances}
& Structural framework, not a replacement sampler \\
Markovization gap measures Markov-compression difficulty
& Projection identity, MMSE bottleneck interpretation, and Gaussian closed-form calculation
& Population quantity; proxy estimators may be biased \\
Lower proxy gap ranks easier coupling choices
& Toy, Digits/LFW latent ODE, CIFAR-10, and Fashion-MNIST pilots
& Pilot-scale evidence under fixed budgets \\
The diagnostic can vary both bridge and coupling
& CIFAR-10 straight/Brownian $\times$ independent/minibatch OT proxy table
& Brownian rows are diagnostic-only, not downstream image models \\
Proxy ranking is robust to simple estimator changes
& CIFAR-10 PCA/kNN ranking check in Table~\ref{tab:proxy_gap_robustness}
& Absolute proxy values are not certified pixel-space gaps \\
\bottomrule
\end{tabular}
\end{table}

\section{Anonymous release and reproducibility checklist for the code}
\label{app:code_reproducibility}

The supplementary code should expose six entry points: (i) generation of the single-run curve in Figure~\ref{fig:toy}; (ii) replication over seeds and sample sizes; (iii) the nearest-neighbor sensitivity table and CIFAR-10 robustness table in Appendix~\ref{app:additional_diagnostics}; (iv) the end-to-end latent Markov-decoder benchmark in Table~\ref{tab:end_to_end_benchmark}; (v) the CIFAR-10 proxy-gap pilot in Table~\ref{tab:2x2_gap} and the CIFAR-10 rows of Table~\ref{tab:pixel_pilots}; and (vi) the Fashion-MNIST rows of Table~\ref{tab:pixel_pilots}. The corresponding scripts are
\begin{itemize}[leftmargin=2em,itemsep=0pt,topsep=2pt]
    \item \texttt{scripts/toy\_projection\_gap\_replicates.py}
    \item \texttt{scripts/run\_realdata\_gap.py}, \texttt{scripts/run\_realdata\_benchmark.py}
    \item \texttt{scripts/run\_end\_to\_end\_compressibility.py}
    \item \texttt{scripts/estimate\_proxy\_gap.py}, \texttt{scripts/bridge\_coupling\_2x2\_gap.py}
    \item \texttt{scripts/run\_pixel\_fm\_benchmark.py}
\end{itemize}
The synthetic experiment has no learned parameters and requires only NumPy, SciPy, scikit-learn, and Matplotlib. The end-to-end experiment additionally uses PyTorch for the small MLP velocity decoder. The OT coupling uses the exact Hungarian assignment on each minibatch; no Sinkhorn regularization or stochastic optimizer is used for the coupling. This makes the diagnostic and latent end-to-end runs reproducible on CPU in minutes at the reported sample sizes.

For anonymous review, the repository should avoid institution-specific paths, author names, non-anonymous commit history, and links to public author-identifying repositories. The supplementary ZIP should use anonymous filenames and a README that states how to reproduce each figure without revealing author identity during the review period.

\section{What would make the framework a full model class?}
\label{app:model_class}

The present paper mostly fixes or externally chooses \(\Gamma\) and \(\scrR\). A fully learnable Bridge Graphical Model would optimize
\begin{equation}
    \max_{\phi,\lambda,\theta,\psi}
    \E_{\pdata(x)}\E_{\scrQ_{\phi,\lambda}(\cdot\mid x)}
    \left[\log p_\psi(x\mid Y_1)-\log\frac{\dd \scrQ_{\phi,\lambda}(\cdot\mid x)}{\dd \scrP_\theta}(Y_{0:1})\right],
\end{equation}
with the Radon--Nikodym term interpreted through Girsanov in support-compatible stochastic dynamics representations or through matching/optimal-control relaxations in deterministic representations. For exact endpoint bridges, this optimization must use a truncated/noisy-terminal version of the path KL as discussed in Section~\ref{sec:bgm} and Appendix~\ref{app:elbo}. Three computational difficulties then appear. First, sampling from \(\scrR_\lambda^{z,x}\) must remain cheap enough for minibatch training. Second, the local target \(u_t\) or \(\beta_t\) must be computable or estimable without solving a global bridge problem at every step. Third, the diffusion representation \(a_t\) should be chosen for numerical stability while preserving the learned current.

Concrete instantiations differ by which approximation is acceptable. A shared-covariance SDE gives the cleanest variational derivation but requires stochastic simulation and score-like correction terms. A deterministic flow gives a simple regression target but loses exact path-space absolute continuity. A small-noise field-biased bridge can regularize singular field lines but introduces an annealing schedule and a conditioned diffusion sampler. These are not merely engineering choices: they change the path law, the Markovization gap, and the statistical difficulty of estimating the projected current.

These difficulties are also opportunities. Learning \(\Gamma_\phi\) imports representation learning into diffusion/flow training; learning \(\scrR_\lambda\) imports path-posterior design; learning or adapting the dynamics representation imports sampler design. The BGM decomposition makes these three choices explicit and therefore makes it possible to ablate them independently.

\section{Technical comparison with closest unification frameworks}
\label{app:comparison}

Generator Matching provides the closest general Markov-process formalism. It treats the infinitesimal generator \(L_t\) as the learned object, derives marginal generators from data-conditioned generators, and supports flows, diffusions, jump processes, Markov superpositions, and multimodal product-state constructions \citep{holderrieth2025generator}. Recent statistical notes analyze the same generator-matching viewpoint through probability paths, infinitesimal generators, Bregman objectives, stability, and statistical rates \citep{aamari2026statistical}. BGMs are complementary: they do not try to characterize the whole generator design space. They isolate the endpoint-coupled bridge law that produces the conditional training target and quantify the residual error incurred when that target is compressed into a Markov decoder. In the flow case, this residual is exactly the Markovization gap.

Stochastic interpolants specify an interpolating random variable, often written abstractly as \(I_t(x_0,x_1,\xi)\), and then learn the velocity or score associated with the induced marginal path. In BGM language, this corresponds to choosing a coupling \(\Gamma\) and a bridge kernel \(\scrR^{z,x}\) whose samples are \(X_t=I_t(z,x,\xi)\). The BGM decomposition is therefore compatible with stochastic interpolants, but it makes two additional degrees of freedom explicit: the endpoint coupling can be learned as an encoder, and the same marginal path can later be represented in different dynamics representations. This distinction matters when one wants to separate representation learning from sampler design.

Diffusion Flow Matching is even closer: it explicitly studies coupling, deterministic or stochastic bridges, and Markovian projection \citep{silveri2024dfm}. IPMF and related Schr\"odinger-bridge solvers also rely on alternating projections involving Markovian and reciprocal structure \citep{kholkin2026ipmf}. The extra object introduced here is the Markovization gap as a scalar diagnostic of how lossy that projection is. The projection identity alone says what the optimal velocity is; the gap says how much endpoint-conditioned bridge information cannot be recovered by any Markov vector field. This is the quantity estimated in the toy experiment and analyzed in closed form in Appendix~\ref{app:gaussian_gap}.

Latent stochastic interpolants add an encoder and decoder around stochastic interpolants. BGM adopts the same variational intuition but separates the bridge kernel from the dynamics representation. This separation is useful because a learned latent path posterior may be sampled as a deterministic ODE, a stochastic SDE, or an augmented field without changing the endpoint coupling. In other words, the encoder answers ``which latent endpoint and bridge path explain this datapoint?'', while the dynamics representation answers ``how should the same current be represented numerically?''

Field and interaction matching approaches begin from a different primitive: an augmented potential or interaction field rather than a time-indexed interpolation. BGM imports these models by constructing a hitting map and field-line bridge kernel. This is not simply a notational change. It gives field methods an endpoint coupling, a bridge law, and a Markovization gap, making them comparable to diffusion and flow methods with the same diagnostic. The field-line Markovization gap is the field analogue of the projection loss: it measures when an augmented field line carries endpoint information that collapses under a Markov velocity on the observed space.

The comparison also clarifies the scope of novelty. BGM does not introduce a fundamentally new Markov projection theorem, a new Schr\"odinger bridge solver, or a new field equation. Its contribution is a common factorization plus a diagnostic that can be applied across these existing families. This makes it possible to compare a Brownian bridge, a deterministic OT bridge, and a potential-field bridge on the same question: how much of the endpoint-conditioned local motion survives compression to a Markov decoder?

\section{Potential-field bridge details}
\label{app:potential_details}

A Poisson-style field bridge can be described in terms of a signed or nonnegative source measure \(\nu\) on an augmented domain \(\Omega\subset\R^{d+D}\). A potential \(\Phi\) solves, in the weak sense,
\begin{equation}
    -\Delta \Phi = c_{d,D}\nu,
\end{equation}
with boundary conditions chosen by the model class. The field is \(\F=-\nabla\Phi\). If \(S_0\) is a simple source surface and \(S_1\) is the data surface, then field-line generation starts from \(Z\sim\mu_0\) on \(S_0\), follows \(\dot y_s=\F(y_s)\), and stops when the trajectory hits \(S_1\). The hitting map \(T_\Phi:S_0\to S_1\) then induces the coupling \((\Id,T_\Phi)_\#\mu_0\).

The time parameter used by a generative model is not necessarily the physical field-line time \(s\). A clock \(r_z:[0,1]\to[0,\tau(z)]\) creates the bridge path \(\gamma_z(t)=\varphi_{r_z(t)}(z)\). Different clocks preserve the same geometric field lines but change the velocity target \(u_t(z)=r'_z(t)\F(\gamma_z(t))\) and therefore change the Markovization gap. This is another bridge-design degree of freedom: two models can share the same potential field but induce different regression problems because they traverse the field at different speeds.

Singularities are unavoidable when data are represented as charges. The deterministic bridge is well-defined only away from singular sets and only for trajectories that hit the data surface transversally. The small-noise Doob bridge in Appendix~\ref{app:doob} regularizes this by replacing deterministic field lines with field-biased stochastic paths conditioned on endpoints. This yields a continuum between an electrostatic bridge and an entropy-regularized bridge, and it suggests a practical path for training: use small noise to avoid singular field-line crossings early in training, then anneal toward sharper field-line transport.

\section{Real-data latent benchmark details}
\label{app:realdata}

The real-data benchmarks use the same BGM components as the synthetic diagnostic but replace the target distribution with real-image latents. Images are flattened into pixels in $[0,1]$, randomly split into train/test with test fraction $0.25$ and seed $0$, standardized using the training-set mean and standard deviation, and embedded with PCA whitening fitted only on the training split. Target endpoint sets are sampled with replacement from the training latents; terminal generated samples are compared against held-out test latents. Minibatch OT uses the full sampled batch of size $n$ for each seed and solves the Hungarian assignment with squared Euclidean cost.

We report two kinds of real-data diagnostics. First, the lightweight ridge benchmark reports integrated Markovization gap, fixed-model velocity-regression error, and a latent two-sample metric. Ridge regression uses penalty $\lambda_{\mathrm{ridge}}=10^{-4}$ and features $(X_t,\,tX_t,\,t,\,t^2,\,\sin\pi t,\,\cos\pi t,\,1)$. The flag \texttt{--ridge-repeats 8} means that each endpoint pair is repeated with eight independent bridge times for fitting; the reported ridge eval MSE uses the same endpoint pairs with fresh times, not a held-out endpoint split. Latent MMD$^2$ is the biased RBF MMD with bandwidth chosen by the median pairwise-distance heuristic on the union of generated and held-out target latents. Second, the end-to-end benchmark trains a neural MLP velocity decoder, integrates its ODE from fresh base samples, and evaluates terminal latent samples. The exact commands used for the LFW ridge subset are:
\begin{verbatim}
PYTHONPATH=src python -B scripts/run_realdata_gap.py \
  --dataset lfw --n 150 --seeds 0 1 2 --k-neighbors 16
PYTHONPATH=src python -B scripts/run_realdata_benchmark.py \
  --dataset lfw --n 150 --seeds 0 1 2 --k-neighbors 16 --ridge-repeats 8
\end{verbatim}
For the end-to-end MLP benchmark, the velocity eval MSE is computed on the training endpoint pairs with fresh bridge times after 500 optimization steps; terminal MMD$^2$ and 1NN two-sample accuracy use samples generated from fresh standard-normal base points and compare to held-out target latents. The end-to-end results in Table~\ref{tab:end_to_end_benchmark} were generated with:
\begin{verbatim}
PYTHONPATH=src python -B scripts/run_end_to_end_compressibility.py \
  --dataset digits --n 500 --eval-n 500 --seeds 0 1 2 \
  --steps 500 --batch-size 256 --hidden-dim 96 --lr 0.001 \
  --k-neighbors 24 --ode-steps 16 32 64
PYTHONPATH=src python -B scripts/run_end_to_end_compressibility.py \
  --dataset lfw --n 150 --eval-n 150 --seeds 0 1 2 \
  --steps 500 --batch-size 128 --hidden-dim 96 --lr 0.001 \
  --k-neighbors 16 --ode-steps 16 32 64
\end{verbatim}
The same real-data scripts accept \texttt{--dataset mnist}, \texttt{--dataset fashionmnist}, and \texttt{--dataset cifar10}; those runs require local torchvision data download. The benchmarks are deliberately latent and low-compute. They evaluate the BGM compressibility diagnostic and a small Markov decoder, not image-generation state of the art.

\begin{figure}[h]
\centering
\begin{minipage}[t]{0.48\linewidth}
\centering
\includegraphics[width=\linewidth]{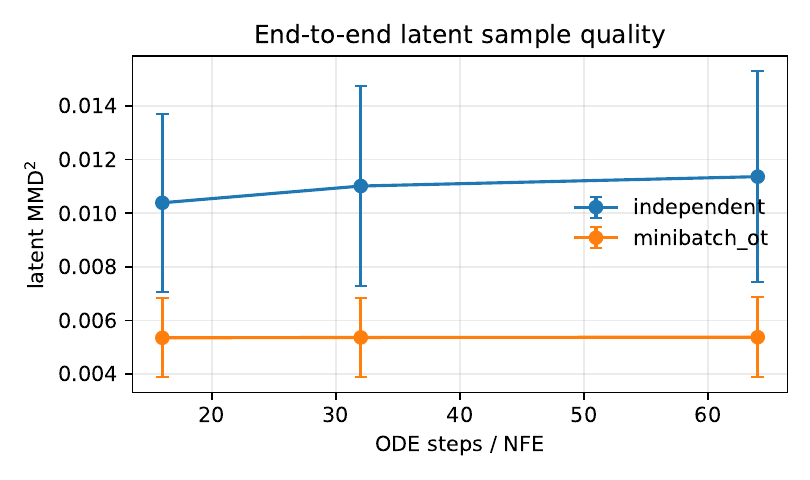}
\end{minipage}\hfill
\begin{minipage}[t]{0.48\linewidth}
\centering
\includegraphics[width=\linewidth]{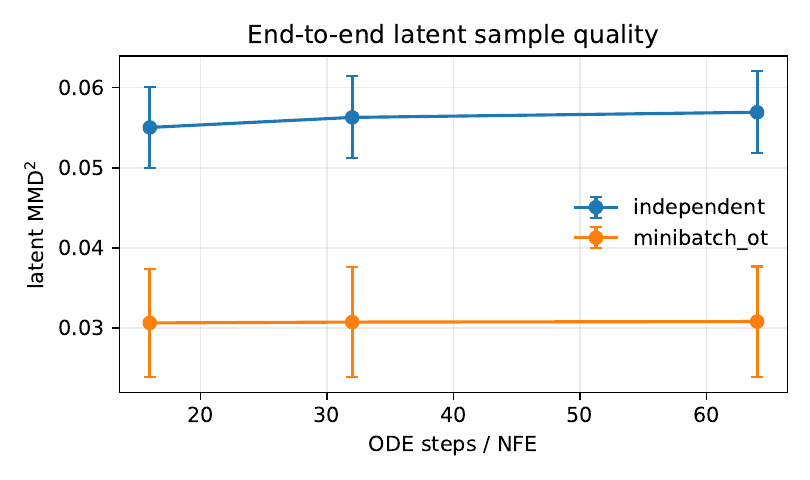}
\end{minipage}
\caption{Terminal latent MMD$^2$ versus Euler steps for the end-to-end Markov decoder benchmark. The trend is mostly insensitive to NFE at this small scale; the coupling effect dominates the discretization effect.}
\label{fig:end_to_end_mmd}
\end{figure}

\section{CIFAR-10 proxy-gap pilot details}
\label{app:cifar_proxy}

The CIFAR-10 pilot is designed to test the diagnostic claim rather than to compete with full-scale image generators. The proxy gap in Table~\ref{tab:2x2_gap} is estimated before neural training in a low-dimensional linear feature space: images are downsampled to $8\times8$ features, projected to 64 PCA dimensions, paired with standard-normal base samples, and evaluated with the same nearest-neighbor conditional-variance plug-in estimator used in the lower-dimensional diagnostics. The reported configuration uses $n=1024$, $k=32$, seeds $0,1,2$, 19 time points in $[0.05,0.95]$, minibatch OT batch size 128, and Brownian variance parameter $\epsilon=0.5$. The Brownian rows use the same feature-space proxy estimator with Brownian bridge perturbations and are included only to test whether the proxy ranking changes under a second bridge law; no Brownian image model was trained.

The straight-bridge CIFAR rows reuse the 20k-step pixel-space flow-matching pilot. Both couplings use the same CIFAR-10 training data, compact U-Net velocity model, optimizer, batch size, Euler ODE sampler, and 10k generated images for FID/KID/Inception-score evaluation. The final training loss is the last logged flow-matching loss, and FID is computed with torch-fidelity against 10k reference CIFAR-10 images. The corresponding commands are in the anonymized supplementary package; the table-generation utilities write \texttt{bridge\_coupling\_2x2\_results.json}, the diagnostic-only 2$\times$2 proxy table, and the separate straight-bridge downstream pilot table. Because only two straight-bridge configurations have downstream image metrics, we report rank agreement rather than a formal correlation coefficient. The robustness table in Appendix~\ref{app:additional_diagnostics} recomputes the proxy under PCA32/PCA64 and $k\in\{16,32,64\}$; the ranking is stable in all tested configurations.

\paragraph{Fashion-MNIST pixel-space pilot.}
The Fashion-MNIST rows of Table~\ref{tab:pixel_pilots} use the same straight-bridge flow-matching setup and evaluation protocol as the CIFAR-10 straight-bridge pilot, with 20k training steps per coupling and 10k generated/reference images for FID/KID/Inception-score evaluation. The reported train loss is the mean of the last 100 logged losses. The pre-training proxy gap is again used only for within-dataset ranking; its absolute scale is not comparable to CIFAR-10 because the feature map and data scale differ. The diagnostic took about three minutes on CPU, while the two downstream training/evaluation runs took about 30 minutes on GPU.

\section{Ethics, existing assets, new assets, and compute}
\label{app:ethics_assets}

\paragraph{Existing assets.}
The reported real-data experiments use scikit-learn handwritten digits, the LFW face subset~\citep{huang2007lfw} distributed through scikit-image~\citep{vanderwalt2014skimage}, Fashion-MNIST~\citep{xiao2017fashionmnist}, and CIFAR-10~\citep{krizhevsky2009learning}. LFW is embedded only into low-dimensional PCA latents for diagnostic regression and latent ODE sampling tasks; it is not used for identification, recognition, or deployment. CIFAR-10 and Fashion-MNIST are used for controlled proxy-gap and pilot flow-matching experiments. The software dependencies include NumPy, SciPy, scikit-learn~\citep{pedregosa2011sklearn}, scikit-image~\citep{vanderwalt2014skimage}, Matplotlib, PyTorch, torchvision, and torch-fidelity for optional image metrics. Because LFW contains public web photographs whose image copyrights remain with original rights holders, the source package does not redistribute raw LFW images. It also does not redistribute raw MNIST, Fashion-MNIST, or CIFAR-10 datasets; optional scripts download them through the user's local environment when explicitly requested. The released code and generated diagnostic outputs are provided under the repository license.

\paragraph{New assets.}
The new assets are the anonymized source code, LaTeX source, generated diagnostic figures, CSV/JSON summaries, and optional small sample grids. The released code is intended to reproduce the Markovization-gap, latent-regression, small end-to-end latent ODE diagnostics, and CIFAR-10/Fashion-MNIST proxy-gap pilots; no large pretrained generator, large model checkpoint, or new dataset is released.

\paragraph{Compute.}
The synthetic and latent diagnostics are CPU-only. The synthetic diagnostic runs in seconds to minutes on a standard laptop-class CPU. The LFW ridge benchmark uses $n=150$, three seeds, a 16-dimensional PCA embedding, nearest-neighbor conditional-variance estimates, and a closed-form ridge model. The end-to-end latent ODE experiments use three seeds, 500 MLP training steps per coupling, and small batches ($n=500$ for digits and $n=150$ for LFW); they also run without a GPU. The CIFAR-10 proxy-gap estimates in Table~\ref{tab:2x2_gap} run on CPU in roughly nine minutes per configuration, and the Fashion-MNIST proxy estimate in Table~\ref{tab:pixel_pilots} takes about three minutes on CPU. The CIFAR-10 and Fashion-MNIST pilot trainings each use one GPU for 20k steps per coupling and 10k generated images for FID/KID/Inception-score evaluation; these are pilot-scale runs rather than full high-compute benchmarks.

\paragraph{Broader impacts and safeguards.}
The positive impact is conceptual and methodological: the framework may help researchers compare bridge, coupling, and dynamics-representation choices before expensive generative training, potentially reducing compute waste and making design decisions more interpretable. The main negative impact is indirect: improved generative modeling methods can contribute to dual-use image generation, including misleading synthetic media, impersonation, and non-consensual image manipulation, if later scaled into high-fidelity generators. This paper does not release a deployable generator or pretrained model; the reproducible release is limited to diagnostics, proof code, and low-dimensional latent experiments. If the framework is later used in high-fidelity generative systems, the appropriate safeguards would be inherited from that application context, including dataset governance, consent-aware data curation, watermarking or provenance tools where applicable, and misuse evaluation for synthetic-media generation.

\section{Suggested NeurIPS ablation plan}
\label{app:future_ablations}

The present submission is intentionally Concept \& Feasibility, but the framework suggests a direct ablation ladder. The first ablation fixes the bridge to be straight and varies only the coupling: independent, minibatch OT, entropic OT, and encoder-induced coupling. The second fixes the coupling and varies the bridge: straight deterministic, Brownian, small-noise field-biased, and Schr\"odinger-like. The third fixes the learned current and varies the dynamics representation: ODE, isotropic SDE with score correction, and augmented field. The fourth learns \(q_\phi(z\mid x)\) jointly with a latent bridge, testing whether the endpoint encoder lowers the estimated Markovization gap and improves downstream sample quality.

The key experimental prediction is not merely that OT or learned couplings produce nicer paths. The end-to-end latent benchmark in Section~\ref{sec:toy} tests the first step of this prediction: lower estimated Markovization gap correlates with lower neural velocity error and lower terminal latent MMD under controlled architecture and compute. The CIFAR-10 and Fashion-MNIST pilots add a second step: a proxy gap estimated before training ranks coupling choices across bridge laws and datasets, and the lower-gap rows have lower training loss and lower pilot FID under fixed budgets, with KID improving on Fashion-MNIST and remaining inconclusive on CIFAR-10. The next falsification test is stronger: if two bridge choices have very different \(\mathfrak G\) but identical convergence speed, solver cost, and image-quality metrics under controlled architectures at larger scale, then the diagnostic is incomplete. Conversely, if \(\mathfrak G\) predicts optimization difficulty across datasets and architectures, it becomes a useful model-selection tool before expensive generative training.

\clearpage
\section*{NeurIPS Paper Checklist}

The checklist is designed to encourage best practices for responsible machine learning research, addressing issues of reproducibility, transparency, research ethics, and societal impact. Do not remove the checklist: {\bf The papers not including the checklist will be desk rejected.} The checklist should follow the references and follow the (optional) supplemental material. The checklist does NOT count towards the page
limit.

Please read the checklist guidelines carefully for information on how to answer these questions. For each question in the checklist:
\begin{itemize}
    \item You should answer \answerYes{}, \answerNo{}, or \answerNA{}.
    \item \answerNA{} means either that the question is Not Applicable for that particular paper or the relevant information is Not Available.
    \item Please provide a short (1--2 sentence) justification right after your answer (even for \answerNA).
\end{itemize}

{\bf The checklist answers are an integral part of your paper submission.} They are visible to the reviewers, area chairs, senior area chairs, and ethics reviewers. You will also be asked to include it (after eventual revisions) with the final version of your paper, and its final version will be published with the paper.

The reviewers of your paper will be asked to use the checklist as one of the factors in their evaluation. While \answerYes{} is generally preferable to \answerNo{}, it is perfectly acceptable to answer \answerNo{} provided a proper justification is given (e.g., ``error bars are not reported because it would be too computationally expensive'' or ``we were unable to find the license for the dataset we used''). In general, answering \answerNo{} or \answerNA{} is not grounds for rejection. While the questions are phrased in a binary way, we acknowledge that the true answer is often more nuanced, so please just use your best judgment and write a justification to elaborate. All supporting evidence can appear either in the main paper or the supplemental material, provided in appendix. If you answer \answerYes{} to a question, in the justification please point to the section(s) where related material for the question can be found.


\begin{enumerate}

\item {\bf Claims}
    \item[] Question: Do the main claims made in the abstract and introduction accurately reflect the paper's contributions and scope?
    \item[] Answer: \answerYes{} 
    \item[] Justification: The abstract and Introduction state that this is a Concept \& Feasibility framework with proof-level diagnostics, synthetic and latent end-to-end tests, and controlled CIFAR-10/Fashion-MNIST pixel-space pilots, not a state-of-the-art image-generation benchmark. The main claims are supported by Sections~\ref{sec:bgm}--\ref{sec:toy} and qualified in Section~\ref{sec:limitations}.

\item {\bf Limitations}
    \item[] Question: Does the paper discuss the limitations of the work performed by the authors?
    \item[] Answer: \answerYes{} 
    \item[] Justification: Section~\ref{sec:limitations} discusses the limited empirical scale, assumptions in the stability theorem, high-dimensional estimation issues, and field-line regularity assumptions. Appendices~\ref{app:future_ablations} and~\ref{app:ethics_assets} further delimit what the current experiments and release do not claim.

\item {\bf Theory assumptions and proofs}
    \item[] Question: For each theoretical result, does the paper provide the full set of assumptions and a complete (and correct) proof?
    \item[] Answer: \answerYes{} 
    \item[] Justification: The theorem and proposition statements in Sections~\ref{sec:markovization},~\ref{sec:dynamics_representation}, and~\ref{sec:field} state their assumptions, with proof details in Appendices~\ref{app:setup}--\ref{app:doob}. Appendix~\ref{app:stability} gives the stability proof and Appendix~\ref{app:gaussian_gap} gives the closed-form Gaussian gap results.

    \item {\bf Experimental result reproducibility}
    \item[] Question: Does the paper fully disclose all the information needed to reproduce the main experimental results of the paper to the extent that it affects the main claims and/or conclusions of the paper (regardless of whether the code and data are provided or not)?
    \item[] Answer: \answerYes{} 
    \item[] Justification: Section~\ref{sec:toy} describes the synthetic diagnostics, real-data latent end-to-end experiments, and CIFAR-10/Fashion-MNIST pixel-space pilots, while Appendices~\ref{app:experiments},~\ref{app:additional_diagnostics},~\ref{app:realdata}, and~\ref{app:cifar_proxy} give distributions, sample sizes, seeds, estimators, commands, and evaluation metrics. Appendix~\ref{app:code_reproducibility} states the scripts needed to reproduce figures and tables.

\item {\bf Open access to data and code}
    \item[] Question: Does the paper provide open access to the data and code, with sufficient instructions to faithfully reproduce the main experimental results, as described in supplemental material?
    \item[] Answer:\answerYes{} 
    \item[] Justification: Appendices~\ref{app:code_reproducibility},~\ref{app:realdata}, and~\ref{app:cifar_proxy} provide exact commands and protocols for reproducing the diagnostic and pixel-space pilot results, and the supplementary repository contains the scripts and generated CSV/PDF/JSON outputs. Appendix~\ref{app:ethics_assets} states that raw external datasets are not redistributed; optional real-data scripts download data locally when explicitly requested.

\item {\bf Experimental setting/details}
    \item[] Question: Does the paper specify all the training and test details (e.g., data splits, hyperparameters, how they were chosen, type of optimizer) necessary to understand the results?
    \item[] Answer: \answerYes{} 
    \item[] Justification: Section~\ref{sec:toy} specifies the bridge, couplings, latent dimension, PCA whitening, pixel-space pilots, fixed models, and reported metrics. Appendices~\ref{app:experiments},~\ref{app:realdata},~\ref{app:additional_diagnostics}, and~\ref{app:cifar_proxy} give the time grid, nearest-neighbor estimator, seeds, sample sizes, training steps, evaluation protocol, and sensitivity settings.

\item {\bf Experiment statistical significance}
    \item[] Question: Does the paper report error bars suitably and correctly defined or other appropriate information about the statistical significance of the experiments?
    \item[] Answer: \answerYes{} 
    \item[] Justification: Figure~\ref{fig:toy} reports five-seed synthetic replications, Table~\ref{tab:end_to_end_benchmark} reports mean $\pm$ standard deviation over three real-data latent seeds, and the CIFAR-10/Fashion-MNIST pilot tables report controlled single-run pixel-space pilot metrics with clear scope limitations. Appendix~\ref{app:additional_diagnostics} gives sensitivity and robustness results for the proxy estimator.

\item {\bf Experiments compute resources}
    \item[] Question: For each experiment, does the paper provide sufficient information on the computer resources (type of compute workers, memory, time of execution) needed to reproduce the experiments?
    \item[] Answer: \answerYes{} 
    \item[] Justification: Appendix~\ref{app:ethics_assets} gives compute details: synthetic and latent diagnostics are CPU-only, CIFAR-10 proxy estimates take roughly nine minutes per configuration on CPU, Fashion-MNIST proxy estimates take about three minutes on CPU, and the CIFAR-10/Fashion-MNIST pixel-space pilots use one GPU for 20k training steps per coupling plus 10k-sample FID/KID/Inception-score evaluation. Appendix~\ref{app:code_reproducibility} describes the scripts needed to reproduce these runs.

\item {\bf Code of ethics}
    \item[] Question: Does the research conducted in the paper conform, in every respect, with the NeurIPS Code of Ethics \url{https://neurips.cc/public/EthicsGuidelines}?
    \item[] Answer: \answerYes{} 
    \item[] Justification: The work is a theoretical and diagnostic study, does not release a deployable generator, and does not conduct new human-subject or crowdsourcing studies. Appendix~\ref{app:ethics_assets} describes the use of public image assets, non-identification latent processing, and the absence of model deployment.

\item {\bf Broader impacts}
    \item[] Question: Does the paper discuss both potential positive societal impacts and negative societal impacts of the work performed?
    \item[] Answer:  \answerYes{}
    \item[] Justification: Appendix~\ref{app:ethics_assets} discusses the positive methodological impact of more interpretable and efficient bridge selection, and the indirect negative dual-use risk that improved generative modeling can contribute to misleading synthetic media if later scaled. Section~\ref{sec:limitations} also limits the current claims to diagnostics rather than deployment.

\item {\bf Safeguards}
    \item[] Question: Does the paper describe safeguards that have been put in place for responsible release of data or models that have a high risk for misuse (e.g., pre-trained language models, image generators, or scraped datasets)?
    \item[] Answer: \answerNA{} 
    \item[] Justification: Appendix~\ref{app:ethics_assets} states that the release does not include a trained image generator, large model checkpoint, or new scraped dataset. The released assets are diagnostic code, generated CSVs/figures, and LaTeX source, so high-risk model-release safeguards are not applicable.

\item {\bf Licenses for existing assets}
    \item[] Question: Are the creators or original owners of assets (e.g., code, data, models), used in the paper, properly credited and are the license and terms of use explicitly mentioned and properly respected?
    \item[] Answer: \answerNo{} 
    \item[] Justification: Appendix~\ref{app:ethics_assets} credits the LFW subset, Fashion-MNIST, CIFAR-10, and software packages, and the source package does not redistribute raw datasets. We answer \answerNo{} rather than \answerYes{} because the copyright status of individual LFW web photographs is not fully explicit; this should be verified or the default benchmark should be switched to a clearer-license dataset before final submission.

\item {\bf New assets}
    \item[] Question: Are new assets introduced in the paper well documented and is the documentation provided alongside the assets?
    \item[] Answer: \answerYes{} 
    \item[] Justification: Appendices~\ref{app:code_reproducibility} and~\ref{app:ethics_assets} document the new code, generated figures, generated CSVs, intended use, and release scope. The supplementary README gives compile and rerun commands, and no new dataset, large pretrained generator, or large model checkpoint is introduced.

\item {\bf Crowdsourcing and research with human subjects}
    \item[] Question: For crowdsourcing experiments and research with human subjects, does the paper include the full text of instructions given to participants and screenshots, if applicable, as well as details about compensation (if any)?
    \item[] Answer:  \answerNA{}
    \item[] Justification: The paper does not conduct crowdsourcing, recruit participants, or collect new human-subject data. Sections~\ref{sec:toy} and Appendices~\ref{app:realdata} and~\ref{app:cifar_proxy} use synthetic data, scikit-learn digits, a public LFW subset for non-identification latent diagnostics, and standard Fashion-MNIST/CIFAR-10 image benchmarks for controlled pilots.

\item {\bf Institutional review board (IRB) approvals or equivalent for research with human subjects}
    \item[] Question: Does the paper describe potential risks incurred by study participants, whether such risks were disclosed to the subjects, and whether Institutional Review Board (IRB) approvals (or an equivalent approval/review based on the requirements of your country or institution) were obtained?
    \item[] Answer: \answerNA{} 
    \item[] Justification: The paper does not recruit participants, interact with people, or collect new human-subject data; Appendix~\ref{app:ethics_assets} describes the public-data-only release scope. Therefore IRB approval or equivalent review is not applicable to the reported experiments.

\item {\bf Declaration of LLM usage}
    \item[] Question: Does the paper describe the usage of LLMs if it is an important, original, or non-standard component of the core methods in this research? Note that if the LLM is used only for writing, editing, or formatting purposes and does \emph{not} impact the core methodology, scientific rigor, or originality of the research, declaration is not required.
    \item[] Answer: \answerNA{} 
    \item[] Justification: No LLM is used as part of the model, theorem, experiment, data pipeline, or evaluation method; the method is described in Sections~\ref{sec:bgm}--\ref{sec:field} and the experiments in Section~\ref{sec:toy}. If authors used language assistance only for writing, editing, or formatting, that is outside the core methodology under the checklist wording.
\end{enumerate}

\end{document}